\documentclass[11pt]{article}
\usepackage{fullpage}
\usepackage{amsfonts}

\title{From Soft Classifiers to Hard Decisions: How fair can we be?\footnote{The order of authors is alphabetical and does not convey
  any information on the authors' relative contributions.}} 
\author{Ran Canetti\thanks{Boston University and Tel Aviv University. Member of CPIIS. Supported by  NSF awards 1413920 \& 1801564, ISF award 1523/14.} \and Aloni Cohen\thanks{MIT. Supported by NSF award CNS-1413920.} \and Nishanth Dikkala\thanks{MIT. Supported by NSF award CCF-1617730, CCF-1650733, and ONR N00014-12-1-0999.} \and Govind Ramnarayan\thanks{MIT. Supported by NSF award CCF-1665252 and NSF award DMS-1737944.}  \and Sarah Scheffler\thanks{Boston University. Supported by the Clare Boothe Luce Graduate Research Fellowship and NSF award 1414119.}  
\and Adam Smith\thanks{Computer Science Department, Boston University.   Supported in part by NSF awards IIS-1447700 and AF-1763786 and a Sloan Foundation Research Award.  \texttt{ads22@bu.edu}.}}

\usepackage{amsmath,amsthm}
\usepackage{graphicx}
\usepackage[colorinlistoftodos]{todonotes}
\usepackage[font={footnotesize,it}]{caption} 
\usepackage{subfig} 
\usepackage{paralist}
\usepackage[colorlinks=true, , allcolors=black]{hyperref}
\usepackage{ulem} 

\usepackage{xspace}
\newcommand{\mathmode}[1]{\ensuremath{#1}\xspace}
\newcommand{\sfmode}[1]{\textsf{#1}\xspace}

\definecolor{darkgreen}{rgb}{0,0.2,0}
\definecolor{darkgray}{rgb}{0.2,0.2,0.2}
\definecolor{darkbrown}{rgb}{0.2,0,0}

\DeclareMathOperator*{\argmin}{arg\,min}

\newtheorem{theorem}{Theorem}[section]
\newtheorem{lemma}{Lemma}[section]
\newtheorem{fact}{Fact}[section]
\newtheorem{claim}{Claim}[section]
\newtheorem{proposition}{Proposition}[section]
\newtheorem{example}{Example}[section]

\theoremstyle{definition}
\newtheorem{definition}{Definition}[section]

\newcommand{\ty}{\mathmode{Y}}
\newcommand{\tF}{\mathmode{0}} 
\newcommand{\tT}{\mathmode{1}} 
\newcommand{\py}{\mathmode{\hat{y}}}

\newcommand{\pF}{\mathmode{0}} 
\newcommand{\pT}{\mathmode{1}} 
\newcommand{\pX}{\mathmode{\bot}}
\newcommand{\idk}{\pX}
\newcommand{\x}{\mathmode{X}}
\newcommand{\ps}{\mathmode{s}}

\newcommand{\punt}{\idk} 

\newcommand{\E}{\mathmode{\mathop{\mathbb{E}}}}
\newcommand{\G}{\mathmode{\mathcal{G}}}

\newcommand{\X}{\mathmode{\mathcal{X}}}

\newcommand{\BR}[1][]{\mathmode{\sfmode{BR}\ifx\relax#1\relax\else_{#1}\fi}}
\newcommand{\PDR}[1][]{\mathmode{\sfmode{PDR}\ifx\relax#1\relax\else_{#1}\fi}}
\newcommand{\NDR}[1][]{\mathmode{\sfmode{NDR}\ifx\relax#1\relax\else_{#1}\fi}}
\newcommand{\DR}[1][]{\mathmode{\sfmode{DR}\ifx\relax#1\relax\else_{#1}\fi}}
\newcommand{\PPV}[1][]{\mathmode{\sfmode{PPV}\ifx\relax#1\relax\else_{#1}\fi}}
\newcommand{\NPV}[1][]{\mathmode{\sfmode{NPV}\ifx\relax#1\relax\else_{#1}\fi}}
\newcommand{\FPR}[1][]{\mathmode{\sfmode{FPR}\ifx\relax#1\relax\else_{#1}\fi}}
\newcommand{\FNR}[1][]{\mathmode{\sfmode{FNR}\ifx\relax#1\relax\else_{#1}\fi}}
\newcommand{\TPR}[1][]{\mathmode{\sfmode{TPR}\ifx\relax#1\relax\else_{#1}\fi}}
\newcommand{\TNR}[1][]{\mathmode{\sfmode{TNR}\ifx\relax#1\relax\else_{#1}\fi}}
\newcommand{\GFPR}[1][]{\mathmode{\sfmode{GFPR}\ifx\relax#1\relax\else_{#1}\fi}} 
\newcommand{\GFNR}[1][]{\mathmode{\sfmode{GFNR}\ifx\relax#1\relax\else_{#1}\fi}}
\newcommand{\cFPR}[1][]{\mathmode{\sfmode{cFPR}\ifx\relax#1\relax\else_{#1}\fi}}
\newcommand{\cFNR}[1][]{\mathmode{\sfmode{cFNR}\ifx\relax#1\relax\else_{#1}\fi}}

\newcommand{\zo}{\mathmode{\{0,1\}}}

\renewcommand{\x}{\mathmode{X}}
\newcommand{\score}{\ps}
\newcommand{\hard}{\mathmode{\hat{Y}}}

\newcommand{\group}{\mathmode{G}}
\newcommand{\soft}{\mathmode{\hat{S}}}
\newcommand{\supp}{\mathsf{Supp}}
\newcommand{\Supp}{\supp}
\newcommand{\post}{\mathmode{\hat{D}}}
\newcommand{\postsoft}{\mathmode{\post^\mathsf{soft}}}
\newcommand{\decide}{\post}
\newcommand{\sam}{\mathmode{\sim}}
\newcommand{\pdf}{\mathmode{\hat{\mathcal{P}}}}
\newcommand{\DOCS}[1][]{\mathmode{\sfmode{DOCS}\ifx\relax#1\relax\else_{#1}\fi}}
\newcommand{\docs}[1][]{\mathmode{\sfmode{DOCS}\ifx\relax#1\relax\else_{#1}\fi}}
\newcommand{\thresh}{\mathmode{\tau}}

\newcommand{\threshRand}{\mathcal{R}}
\newcommand{\eps}{\mathmode{\epsilon}}

\newcommand{\puntprob}{Q}

\begin{document}

\maketitle

\begin{abstract}
A popular methodology for building binary decision-making classifiers in the presence of imperfect information is to first construct a calibrated non-binary ``scoring" classifier, and then to post-process this score to obtain a binary decision. We study various group fairness properties of this methodology,  when the  non-binary scores  are calibrated over all protected groups, and with a variety of post-processing algorithms. Specifically, we show:
\begin{itemize}
\item
There does not exist a general way to post-process a calibrated classifier to equalize protected groups' positive or negative predictive value (PPV or NPV). For certain ``nice" calibrated classifiers, either PPV or NPV can be equalized when the post-processor uses different thresholds across protected groups. Still, 
when the post-processing consists of a single global threshold across all groups, natural fairness properties, such as equalizing PPV in a nontrivial way, do not hold even for ``nice" classifiers.
\item
When the post-processing stage is allowed to \emph{defer} on some decisions (that is, to avoid making a decision by handing off some examples to a separate process), then  for the non-deferred decisions, the resulting classifier can be made to equalize  PPV, NPV,  false positive rate (FPR) and false negative rate (FNR) across the protected groups.  This suggests a way to partially evade the impossibility results of Chouldechova and Kleinberg et al., which preclude equalizing all of these measures simultaneously.  We also present different deferring strategies and show how they affect the fairness properties of the overall system.
\end{itemize}
We then evaluate our post-processing techniques using the COMPAS data set from 2016.

\end{abstract}


\newpage
{\small
	\tableofcontents
}
\newpage

 \section{Introduction}
The concept of {\em fairness}  is deeply ingrained in our psyche as a fundamental, essential ingredient of Human existence. Indeed the perception of fairness, broadly construed as accepting each others' equal right for well being, is arguably one of the most basic tenets of cooperative societies of individuals in general, and Human existence in particular.

However, as fundamental as this concept may be, it is also extremely elusive:  different societies have developed very different notions of fairness and equality among individuals, subject to different religious, ethical, and social beliefs; in particular, the intricate interplay between fairness and {\em justice,} which is yet another somewhat elusive concept, is often not well-defined and a matter of subjective interpretation.

The concept is further complicated by the fact that human decisions are often made with incomplete information and limited resources. These two factors must be taken into account  when evaluating whether decision-making processes  are ``fair.''
Indeed, these two aspects of the problem have become increasingly prominent as societies grow and decision processes become more complex and algorithmic.

One way that researchers are responding to these growing concerns is by attempting to formulate precise notions for {\em fairness} of decisions processes, e.g.  \cite{dwork,Kleinberg}. While these notions do not intend to capture the complexities of the ethical, socio-economic, or religious aspects of fairness, they do consider the fairness aspects of statistical decision-making processes with incomplete information. Essentially, these notions accept the fact that a decision process with incomplete information will inevitably make errors relative to the desired full-information notion (which is treated as a given), and  provide guidelines on how to ``distribute the errors fairly''  across individuals, or alternatively across groups of individuals. These definitions have proven to be  meaningful and eye opening;  in particular, it has been demonstrated that some very natural notions of ``fair distribution of errors''  are mutually inconsistent: No decision mechanism with incomplete information can satisfy all, except for in trivial cases \cite{Cho17,Kleinberg}.

Faced with this basic impossibility,  we would like to better understand the process of decision making with incomplete information, and use this understanding to propose ways to circumvent this impossibility.

Specifically, we concentrate on the task of post-processing a calibrated soft classifier under group fairness constraints. We suppose that individuals belong to one of two or more disjoint \emph{protected groups}. Our overall task is to decide whether a given individual has some hidden binary property $B$ in a way that ensures  ``fair balancing of errors'' across the groups.

For  that purpose, we consider the following two-stage mechanism. The first stage consists of constructing a
classifier $\soft$ 
that outputs for each individual $x$ a score $\score \in [0,1]$ that is related to the chance that $x$ has property $B$. The only requirement we make of $\soft$ is group-wise calibration: for both $g_1$ and $g_2$, and for each $\score\in[0,1]$, the fraction of individuals in the group that get score $\score$ and have the property, out of all individuals in the group that get score $\score$, is $\score$. The second stage takes as input the output $\score = \soft(x)$ of the first stage and the group to which $x$ belongs, and outputs a binary decision: its best guess at whether $x$ has property $B$.

An attractive aspect of this two-stage mechanism is that each stage can be viewed as aimed at a different goal: The first stage is aimed at gathering information and providing the best accuracy possible, with only minimal regard to fairness (i.e only group-wise calibration). The second stage is aimed to extract a  decision from the  information collected in the first stage, while making sure that the errors are distributed ``fairly.'' 

To further focus our study, we take the first stage as a given and concentrate on the second. That is, we consider the problem of \emph{post-processing} the scores given by the calibrated soft classifier $\soft$ into binary predictions.
A representative example is a judge making a bail decision based on a score provided by a software package.
Following \cite{Cho17,EqualOpp} we consider the following four performance measures for  the resulting binary classifier: the \emph{positive predictive value (PPV)}, namely the fraction of individuals that have the property among all individuals that the classifier predicted  to have the property; The \emph{false positive rate (FPR)}, namely  the fraction of individuals that were predicted to have the property among all individuals that don't have the property; The \emph{negative predictive value (NPV)} and \emph{false negative rate (FNR)}, which are defined analogously.
Ideally, we would like to equalize each one of the four measures across the groups, i.e. the measure will have the same value when restricted to samples from each group.  Unfortunately, however, we know that this is impossible in general \cite{Cho17,Kleinberg}.
This leads us to a broad question that motivates our work:

\begin{quote} 
  Under what conditions can we post-process a calibrated soft classifier's outputs so that the resulting hard
  classifier equates a subset of $\{\PPV, \NPV, \FNR, \FPR\}$ across a set of protected groups? How can we balance these conflicting  goals? 
\end{quote}

\paragraph{Results: Post-Processing With Thesholds.}
In a first set of results we consider the properties obtained by post-processing via  a  ``threshold'' mechanism.
Naively, a threshold post-processing mechanism would return 1  for individual $x$ whenever the calibrated score $\score(x)$ is above some fixed threshold, and return 0 otherwise.  We somewhat extend this mechanism by allowing the post-processor ``fine-tune'' its decision by  choosing the output probabilistically whenever the result of the soft classifier is exactly the threshold.

We first observe that the popular and natural post-processing method of using a single threshold across all groups has an inherent deficiency:  No such mechanism can in general guarantee equality of either PPV or NPV across the protected groups. 

We then show that, when using different thresholds for the different groups, one can equalize \emph{either} PPV or NPV (but not both) across  the two groups, assuming the profile of $\soft$ has some non-degeneracy property.

The combination of the impossibility of single threshold and the possibility of per-group threshold also stands in contrast to the belief that a soft classifier that is calibrated across both groups allows ``ignoring'' group-membership information in any post-processing decision \cite{MP17}. Indeed,   the conversion to a binary decision ``loses information'' in different ways for the two groups, and so  group membership  becomes relevant again after post-processing.

\paragraph{Results: Adding deferrals.}
For the second set of results we consider post-processing strategies that do not always output a decision. Rather, with some probability the output is  $\idk$, or  ``I dont know", which means that the decision is deferred to another (hopefully higher quality, even if more expensive) process.  Let us first present our technical results and  then discuss potential interpretations and context.

The first strategy is a natural extension of the per-group threshold: we use \emph{two} thresholds per group, returning 1 above the right threshold, 0 below the left threshold, and $\idk$ between the thresholds. We show that there always exists a way to choose the  thresholds such that, conditioned
on the decision not being $\idk$, both the PPV and NPV are equal across groups.

Next we show a family of post-processing strategies  where, conditioned on the decision not being $\idk$, {\em all four quantities} (PPV, NPV, FPR, FNR) are equal across groups.

All strategies in this family have the following structure: Given an individual $x$,  the strategy first makes a randomized decision whether to defer on $x$, where the probability depends on $\soft(x)$ and the group membership of $x$.  If not deferred, then the decision is made via another post-processing technique.  

One method for determining the probabilities of  deferrals is to make sure that, the distribution of scores returned by the calibrated soft classifier  \emph{conditioned on not deferring,} is equal for the two groups (That is, let $p_{s,g}$ denote the  probability, restricted to group $g$, that an element gets score $s$ conditioned on not deferring. Then for any $s$,  we choose deferral probabilities so that  $p_{s,g_1}=p_{s,g_2}$.)   The resulting classifier can then be post-processed in \emph{any} group-blind way (say, via a single threshold mechanism as described above).

Of course, the fact that all four quantities are equalized conditioned on not deferring does not, in and of itself, provide any guarantees regarding the fairness properties of the overall decision process --- which includes also  the downstream decision mechanism. For one, it would be naive to simply assume that fairness ``composes'' \cite{DI18}. Furthermore, the impossibility of \cite{Cho17,Kleinberg} says that the overall decision-making process cannot possibly equalize  all four measures. 

However, in some cases one can provide alternative (non-statistical) justification for the fairness of the overall process: For instance, if the downstream decision process never errs, the overall process might  be considered ``procedurally fair.'' We present more detailed reflections on our deferral-based approach in Section~\ref{sec:discussion}.

We note that deferring was considered in machine learning in a number of contexts, including the context of fairness-preservation \cite{MPZ17}. In these works, the classifier typically punts only when its confidence regarding some decision is low. By contrast, we use deferrals in order  to ``equalize'' the probability mass functions of the soft classifier over the two groups, which may involve deferring on individuals for whom  there is higher confidence. Indeed, deferring on some higher-confidence individuals seems inherent to our goal of equalizing PPV, NPV, FPR, and FNR while keeping the deferral rate low.
Furthermore, our framework allows for a wide range of deferral strategies which might be used to promote additional goals. Pursuing alternate strategies for deferral is an interesting direction for future work.

\paragraph{Experimental results.}
We test our methodology on the Broward county dataset with COMPAS scores made public by ProPublica~\cite{angwin2016machine} in order to better understand its strengths and limitations. Indeed, it has been shown that the COMPAS scoring mechanism is an approximately calibrated soft classifier.
We first ran our two-threshold post-processing mechanism and obtained a binary decision algorithm which equalizes both PPV and NPV across Caucasians and African-Americans.

We then ran our post-processing mechanism with deferrals to equalize all four of PPV, NPV, FPR, FNR across the two groups, with three different methods for deciding how to defer:  In the first method, decisions are deferred only for Caucasians; in the second, decisions are deferred only for African Americans; in the third method, decisions are deferred for an equal fraction of Caucasians and of African Americans. This fraction is precisely equal to the statistical  (total variation) distance between the distributions of scores produced by the soft classifier on the two groups. More details about the results along with figures are given in Section~\ref{sec:experiments}.

\paragraph{Extensions and open problems.}
As just mentioned, a natural question is to find alternative ways for deciding when to defer, along with ways to argue fairness properties for the overall combined process.

We also leave open the setting  where individuals belong to multiple, potentially intersecting groups as in \cite{Multicalibration,gerrymandering}.

Yet another question is to consider additional (or alternative) properties of soft classifiers  that will make for more efficient or effective post-processing.

\subsection{Related work}

We briefly describe the works most closely related to ours, though both the list of works and their summaries are inevitably too short. Our work fits in a research program on group fairness notions following the work of Chouldechova \cite{Cho17} and Kleinberg et~al. \cite{Kleinberg}. Those works demonstrate the inherent infeasibility of simultaneously equalizing a collection of measures of group accuracy. Our work considers the notions of calibration as formalized in \cite{Pleiss} and those of PPV, NPV, FPR, and FNR from \cite{Cho17} and \cite{Kleinberg}.

The power of post-processing calibrated scores into decisions using threshold classifiers in the context of fairness has been previously studied by Corbett-Davies, Pierson, Feller, Goel, and Huq \cite{Corbett}. As in our work, they show that it is feasible to equalize certain statistical fairness notions across groups using (possibly different) thresholds. They additionally show that these thresholds are in some sense optimal. Whereas \cite{Corbett} focuses on statistical parity, conditional statistical parity, and false positive rate, our most comparable results consider PPV. In our work, we further show that in some cases thresholds fail to equalize both PPV and NPV (called \emph{predictive parity} by \cite{Cho17}), unless we also allow our post-processor to defer on some inputs. Our work also studies methods of post-processing that are much more powerful than thresholding, especially when allowing deferrals. On the technical side, \cite{Corbett} assumes that their soft classifiers are supported on the continuous interval $[0,1]$, simplifying the analyses. We instead study classifiers with finite support as it is closer to true practice in many settings (e.g., COMPAS risk scores).

Using deferrals to promote fairness
has been considered also in the work of Madras, Pitassi, and Zemel \cite{MPZ17}. Specifically they consider how deferring on some inputs may promote a combination of accuracy and fairness, especially when taking explicit account of the downstream decision maker. They make use of two-threshold deferring post-processors like those discussed in Section 5. While it helped inform our work, \cite{MPZ17} takes a more experimental approach and focuses on minimizing the ``disparate impact,'' a measure of total difference in classification error between groups, while maximizing accuracy. One important difference between our works is that Madras et al.~distinguish between ``rejecting'' and ``deferring.'' Rejecting is oblivious as to properties of the downstream decision maker, while deferring tries to counteract the biases of the decision maker. Our work considers only the former notion, but uses the term ``defer" instead of ``reject."



\section{Preliminaries}
\label{sec:prelims}
We study the problem of binary classification.
An \emph{instance} is an element, usually denoted $x$, of a
universe $\X$.
We restrict our attention to instances sampled uniformly at random
from the universe, denoted $X \sam \X$. Our theory extends directly
to any other distribution on $\X$; that distribution does not need to
be known to the classifiers.
Each instance $x$ is associated with a \emph{true type} $\ty(x) \in \{\tF, \tT\}$.
Each instance $x$ is also associated with a \emph{group} $\group(x) \in \G$, where $\G$ is the set of groups. We restrict our attention to sets $\G$ that form a partition of the universe $\X$.
We denote by $\X_g$ the set of instances $x$ in group $g$, and by $X_g$ the random variable distributed uniformly over $\X_g$. Note that for any events $E_1$ and $E_2$, $\Pr_{X \sam \X_g}[E_1 \mid E_2] = \Pr_{X \sam \X}[E_1 \mid E_2, \group(X) = g]$.

\begin{definition}[Base rate (\BR)]\label{defn:base-rate}
    The \emph{base rate} of a group $g \in \G$, is
    \begin{equation}
	\label{eq:base-rate-def}
    \BR_g = \Pr[\ty(X_g) = \tT] = \E[\ty(X_g)].
    \end{equation}
\end{definition}
\noindent
When $\X$ is finite, $\BR_g$ is simply the fraction of individuals $x$ in
the group $g$ for whom $\ty(x)=\tT$.

A \emph{classifier} is a randomized function with domain $\X\times\G$.\footnotemark
A \emph{hard classifier}, denoted $\hard$, outputs a \emph{prediction} in $\{\pF, \pT\}$, interpreted as a guess of the true type $\ty(x)$.
A \emph{soft classifier}, denoted $\soft$, outputs a \emph{score}  $\score \in [0,1]$, interpreted as a measure of confidence that $\ty(x) = 1$. We restrict our attention to soft classifiers with finite image.
We call a classifier \emph{group blind} if its output is independent of the input group $g$.
For all groups $g \in \G$, we call a hard classifier $\hard$ \emph{non-trivial on $g$} if $\Pr[\hard(X_g) = 1] > 0$ and
$\Pr[\hard(X_g) = 0] > 0$.  Hard classifiers are \emph{trivial on $g$} if they are not non-trivial on $g$.
\footnotetext{As the focus of this paper is on the post processing of classifiers, we set aside questions such as the origin of the given classifier, including the randomness used in training, the origin or quality of the training data, and societal factors affecting the classifier. In particular, the classifiers we consider in this work are memoryless: they do not remember inputs or random choices from previous invocations. That is, we assume that if  $X,X'$ are two independent random variables 
drawn from $\X$ then $\soft(X)$ and $\soft(X')$ (respectively $\hard(X)$ and $\hard(X')$) are also independent 
random variables. (Our formalism can be naturally extended also to classifiers with initial randomized preprocessing, by  considering the family of derivative classifiers, where for each derivative classifier the random choices made at preprocessing are fixed to some value.  The formalism can then be applied separately to each derivative classifier.) 
}

A \emph{post-processor} is a randomized function with domain $[0,1] \times \G$. As with classifiers, a post-processor can be \emph{hard} or \emph{soft}. A hard post-processor, denoted $\post$, outputs a prediction in $\{\pF,\pT\}$. A soft post-processor, denoted $\postsoft$, outputs a score $\score \in [0,1]$. Observe that for a soft classifier $\soft$, $\post\circ\soft$ is a hard classifier, and $\postsoft\circ\soft$ is a soft classifier.
As with classifiers, we call a post-processor group blind if its
output is independent of the group $g$, and we restrict our attention
to post-processors with finite image. The restriction to finite image
is for mathematical convenience (and also because digital
  memory leads to discrete universes);
 our results generalize to infinite images as well.

In Section~\ref{sec:deferrals}, we expand the definitions of both classifier and post-processors to allow an additional input or output: the special symbol $\punt$.

\begin{figure}
    \usetikzlibrary{decorations.text}
    \resizebox{\hsize}{!}{
\begin{tikzpicture}[mypostaction/.style 2 args={
                         decoration={
                              text align={
                                    left indent=#1},
                                    text along path,
                                    text={#2}
                                    },
                           decorate
                        }, auto, node distance=4cm]

                        \node[scale=0.6] [] (X) at (0,0) {$x \in \X$};
                        \node[scale=0.6] (score) at (3,0) {$\score \in [0,1]$};
                        \node[scale=0.6] (decision) at (8,0) {$\py \in \{\pF, \pT\}$};

                        \draw [->] (X.east) -- (score.west) node [above, midway, scale=0.6] {``soft''} node [below, midway, scale=0.6] {$\soft$};
                        \draw [->] (score.east) -- (decision.west) node [above, midway, scale=0.6] {``hard post-processor''} node [below, midway, scale=0.6] {$\decide$};
                        \draw [->] (score.east) -- (decision.west) node [above, midway, scale=0.6] {``hard post-processor''} node [below, midway, scale=0.6] {$\decide$};
    \draw[->] (X.north) to [out=20,in=160] (decision.north) {};
    \draw[->] (2.5, -0.25) to [out=260, in=280, looseness=2] (3.5, -0.25) {};

    \node[scale=0.6] (TEXTpostsoft) at (3.01,-0.6) {$\postsoft$};
    \node[scale=0.6] (TEXTpostsoft) at (3,-1) {``soft post-processor''};

    \node[scale=0.6] (TEXThard) at (4,1.12) {``hard''};
    \node[scale=0.6] (TEXThard) at (4,0.82) {$\hard$};
\end{tikzpicture}
} 
\caption{We call a classifier that returns results in $[0,1]$ a \emph{soft classifier} to differentiate it from those
    which return results in $\{0,1\}$, which we call \emph{hard classifiers}.  We refer to classifiers that take as input
    the output of a soft classifier as \emph{post-processors}.}
\label{fig:notation}
\end{figure}
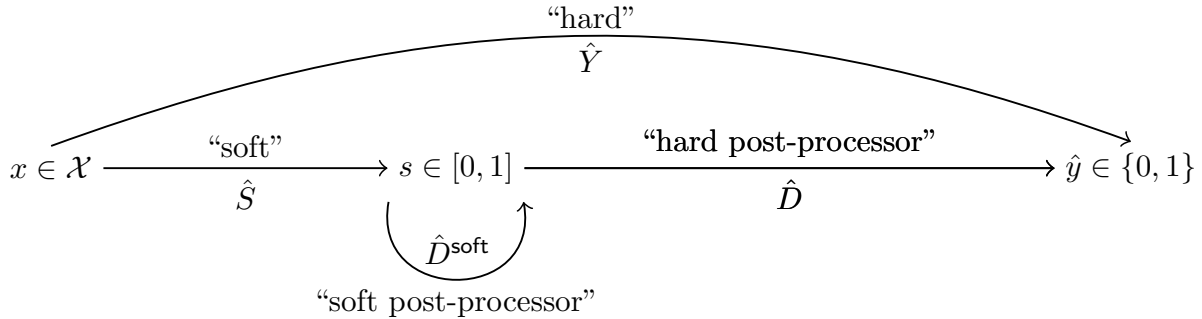

\subsection{Calibration} Several works concerning algorithmic fairness focus on various notions of \emph{calibration}.  The following calibration notions are defined only over soft classifiers:

\begin{definition}[Calibration (Soft)]\label{defn:calibration}
    We say a soft classifier $\soft$ is \emph{calibrated} if $\forall \score \in [0,1]$ for which $\Pr_{X \sam
    \X}[\soft(X) = \score] > 0$,
    \[ \Pr_{X \sam \X}[\ty(X) = \tT \mid \soft(X) = \score] =\E_{X
        \sam \X}[\ty(X) \mid \soft(X) = \score] = \score. \]
The probability above is taken over the sampling of $\x$, as well as
random choices made by $\soft$ at classification time.

\end{definition}

\begin{definition}[Groupwise Calibration (Soft)]\label{defn:groupwise-calibration}
    We say that a soft classifier $\soft$ is \emph{groupwise calibrated} if it is calibrated within all groups.
    That is, $\forall g \in \G$ and $\forall \score \in [0,1]$ for which $\Pr[\soft(X_g) = \score] > 0$, we have that
    \[ \Pr[\ty(X_g) = \tT \mid \soft(X_g) = \score] = \score. \]
\end{definition}

Groupwise calibration is essentially the same notion as {\em multicalibration} \cite{Multicalibration} with the difference that in their case the true types are values in $[0,1]$.  We use a different term to emphasize that we restrict our attention to collections of groups $\G$ that form a partition of the universe $\X$.

The two definitions above are stated for soft classifiers whose
output distribution is discrete, since we must be able to condition on
the event $\soft(X)=\score$ or $\soft(X_g)=\score$. That said, it extends naturally to classifiers
with continuously-distributed outputs provided that the conditional
probabilities are well defined.

\subsection{Accuracy Profiles (APs)}
Throughout this work, we make repeated reference to the probability mass function of the random variable $\soft(X_g)$
for a calibrated soft classifier $\soft$ acting on a randomly distributed input $X_g$.
We call this distribution on calibrated scores an \emph{accuracy profile} (AP).
\begin{definition}[Accuracy Profile (AP)]\label{defn:docs}
    The \emph{accuracy profile (AP)} of a calibrated soft classifier $\soft$ for a group $g$, denoted by $\pdf_g$, is the PMF of $\soft(X_g)$. That is, for $\score \in [0,1]$, $\pdf_g(\score) = \Pr[\soft(X_g) = \score].$
\end{definition}

Abusing notation, we denote by $\pdf$ the collection $\{\pdf_g\}_{g\in \G}$, and call it the \emph{AP of \soft}.
We denote by $\supp(\pdf_g)$ the support of the AP $\pdf_g$, namely the set $\supp(\pdf) = \{\score : \exists x \in
\X_g, \exists r \mbox{ s.t. } \soft(x,r) = \score\} \subseteq [0,1]$.

An accuracy profile is a distribution of scores for a calibrated classifier $\soft$.
Because $\soft$ is calibrated, the AP conveys information
about the performance of $\soft$, and is constrained by properties of
the underlying distribution on $X$. For example, the AP's expectation
is exactly the base rate for the population:
\begin{proposition}
\label{prop:baserate-calibrated}
  For any groupwise calibrated soft classifier $\soft$, for all groups $g \in
  \G$: $\BR_g = \E[\soft(X_g)]$.
\end{proposition}

\begin{proof}[Proof of Proposition~\ref{prop:baserate-calibrated}]
\begin{align*}
\BR_g &= \Pr[\ty(X_g) = \tT] \\
&= \sum_{\score \in \supp(\pdf_g)} \Pr[\ty(X_g) = \tT \mid \soft(X_g) = \score] \Pr[\soft(X_g) = \score] \\
&= \sum_{\score \in \supp(\pdf_g)} \score \Pr[\soft(X_g) = \score] \\
&=  \E[\soft(X_g)]
\end{align*}
where the third line follows from the definition of a calibrated classifier (Definition~\ref{defn:groupwise-calibration}).
\end{proof}

Accuracy profiles also provide useful geometric
intuition for reasoning about the effects of post-processing calibrated
scores. We elaborate on this in Section~\ref{sec:fair-post-process}
(see Figure~\ref{fig:calibration-line-decide}).

\subsection{Group Fairness Measures}

Several well-studied measures of statistical ``fairness''
  (e.g., \cite{EqualOpp,Cho17,Kleinberg, Pleiss,Multicalibration,
    gerrymandering}) look at how the following key performance
  measures of a classifier differ across groups.
The \emph{false positive rate}
(FPR) of a hard classifier $\hard$ for a group $g$ is the rate at which $\hard$ gives a positive classification among
instances $x \in \X_g$ with true type $\tF$. The \emph{false negative rate} (FNR) is defined analogously for predicted
negative instances
with true type $\tT$. \emph{Positive
predictive value} (PPV) and \emph{negative predictive value} (NPV) track the rate of
mistakes within instances that share a predicted type. Informally, positive predictive value captures how much meaning can be given to a predicted $\tT$, and negative predictive value is similar for predicted $\tF$. We now define these statistics formally.

\begin{definition}\label{defn:four-measures} Given a hard classifier $\hard$ and a group $g$, we define\\
\makebox[2.8in]{the \emph{false positive rate} of $\hard$ for $g$: \hfill}  $ \FPR_{\hard,g} = \Pr[\hard(X_g) = \pT \mid \ty(X_g) = \tF] $;\\
\makebox[2.8in]{the \emph{false negative rate} of $\hard$ for $g$: \hfill} $ \FNR_{\hard,g} = \Pr[\hard(X_g) = \pF \mid \ty(X_g) = \tT] $;\\
\makebox[2.8in]{the \emph{positive predictive value} of $\hard$ for
  $g$: \hfill} $ \PPV_{\hard,g} = \Pr[\ty(X_g) = \tT \mid \hard(X_g) =
\pT]$; \\
\makebox[2.8in]{the \emph{negative predictive value} of $\hard$ for $g$: \hfill} $ \NPV_{\hard,g} = \Pr[\ty(X_g) = \tF \mid \hard(X_g) = \pF]. $
\end{definition}

\smallskip

The probability statements in the definitions above
reflect two sources of randomness: the sampling of $X_g$ from the
group $g$ and any random choices made by the classifier $\hard$.

Among previous works, some \cite{EqualOpp,Kleinberg} focus on
equalizing only one or both of the false positive rates and false
negative rates across groups, called \emph{balance} for the negative
and positive classes, respectively.  Equalizing positive and negative
predictive value across groups is often combined into one condition
called \emph{predictive parity} \cite{Cho17}.  We split the value out
to be a separate condition for the positive and negative predictive
classes. 
Predictive parity appears to be a hard-classifier analogue of
calibration: both can be interpreted as saying that the output of the
classifier (hard or soft) contains all the
information contained in group membership. Our results highlight that
the relationship between these notions is more subtle than it first
appears; see Section
\ref{sec:postprocess-limits} for further discussion.



\section{The Limits of Post-Processing}
\label{sec:postprocess-limits}

Suppose throughout this section that $\soft$ is a groupwise calibrated soft classifier.
Our goal in this section is to make binary predictions based on $\soft(x)$ --- and possibly the group $\group(x)$ --- subject to equalizing PPV and/or NPV among groups. That is, we wish to make a prediction using a hard post-processor $\decide$ such that $\hard =\decide \circ \soft$ equalizes
PPV and/or NPV among groups. We chose to concentrate first on (the limitations of) equalizing PPV and NPV rather than FPR and FNR due to the conceptual similarity of PPV and NPV to calibration. Also, the case of equalizing false positive rates with thresholds is addressed in ~\cite{Corbett}.

\subsection{Fairness Conditions for Post-Processors}
\label{sec:fair-post-process}

We begin by making a simple observation about post-processing that provides some geometric intuition for the rest of this section.
Just as in Proposition~\ref{prop:baserate-calibrated}, we can express $\PPV_{\hard, g}$ and $\NPV_{\hard, g}$
succinctly in terms of conditional expectations over the AP $\pdf_g$. 
\begin{proposition}
\label{prop:ppv-as-cond-exp}
    Let $\hard = \post \circ \soft$ be a hard classifier that is non-trivial for all $g \in \G$ where $\soft$
    is groupwise calibrated with respect to $\G$.
    For any $g \in \G$ we have:
    \begin{align*}
    \PPV_{\hard,g} &= \E[\soft(X_g) \mid \hard(X_g) = \pT] \\
    \NPV_{\hard,g} &= 1 - \E[\soft(X_g) \mid \hard(X_g) = \pF]
    \end{align*}
\end{proposition}

\begin{proof}[Proof of Proposition~\ref{prop:ppv-as-cond-exp}]
We first observe that the output of a post-processor is conditionally independent of the
true type, conditioned on the output of the soft classifier it is post-processing and the group membership:

\begin{fact}
\label{fact:conditional-independence}
Consider any randomized function $\post: [0,1] \times \G \to \{0,1\}$. Since
$\post$ is a randomized function with inputs $\score \in [0,1]$ and $g \in \G$, we have that
\begin{equation}
\label{eq:conditional-independence}
(\post(\soft(X),\group(X)) \perp \ty(X)) \mid (\soft(X),\group(X))
\end{equation}
or in other words that $\post(\soft(X), \group(X))$ is conditionally independent of the true type $\ty(X)$, since fixing the inputs to $\post$ makes its output purely a function of its random string.
\end{fact}

Now recall that $\PPV_{\hard,g}$ and $\NPV_{\hard,g}$ are well-defined for all groups because $\hard$ is non-trivial on all
    groups.  We then have
\begin{align*}
\PPV_{\hard, g} &= \Pr[\ty(X_g) = 1 \mid \hard(X_g) = \pT] \\
&= \sum_{\score \in \supp(\pdf_g)} \Pr[\ty(X_g) = 1 , \soft(X_g) = \score \mid \decide(\soft(X_g), g) = \pT] \\
&= \sum_{\score \in \supp(\pdf_g)} \Pr[\ty(X_g) = 1 \mid \soft(X_g) = \score, \decide(\soft(X_g), g)] \\
& \quad \cdot \Pr[\soft(X_g) = \score | \decide(\soft(X_g), g) = \pT] \\
&= \sum_{\score \in \supp(\pdf_g)} \score \Pr[\soft(X_g) = \score | \decide(\soft(X_g), g) = \pT] \\
&= \E[\soft(X_g) \mid \hard(X_g) = \pT]
\end{align*}
where the fourth line follows from the fact that the group $g$ is fixed within $\X_g$, which lets us apply Fact~\ref{fact:conditional-independence}, and the fact that $\soft$ is calibrated on $g$. Similar simplifications give us that
\begin{align*}
\NPV_{\hard, g} &= \Pr[\ty(X_g) = 0 \mid \decide(\soft(X_g),g) = \pF] \\
&= 1 - \Pr[\ty(X_g) = 1 \mid \decide(\soft(X_g),g) = \pF] \\
&= 1 - \E[\soft(X_g) \mid \decide(\soft(X_g), g) = \pF]
\end{align*}
\end{proof}

Using Proposition~\ref{prop:ppv-as-cond-exp}, we can geometrically see how certain post-processing decision rules will interact with the AP for a group $g$.  For example,
using a threshold, the expected true positives, true negatives, false positives, and false negatives can be estimated,
as shown in Figure \ref{fig:calibration-line-decide}.

\begin{figure}
    \centering
\includegraphics[width=0.5\hsize]{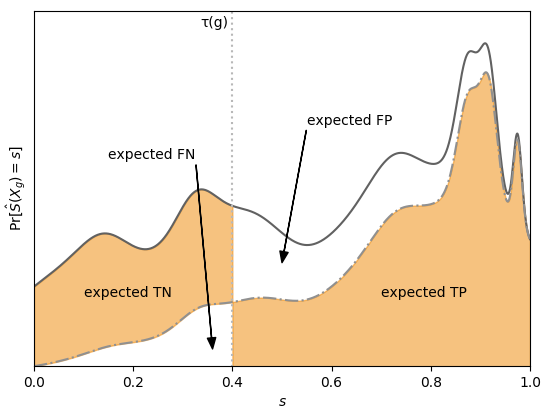}
\caption{Accuracy Profiles (APs, definition \ref{defn:docs}) yield useful geometric intuitions, which come from the calibration
    property (definition \ref{defn:calibration}).
    With a threshold, the expected PPV, NPV, FPR, and FNR can be seen visually.}
\label{fig:calibration-line-decide}
\end{figure}

Proposition~\ref{prop:fpr-as-cond-exp} below gives a characterizations of the false positive and false negative rates in a manner analogous to how Proposition~\ref{prop:ppv-as-cond-exp} describes PPV and NPV:

\begin{proposition}
\label{prop:fpr-as-cond-exp}
Let $\hard$ and $\soft$ be hard and soft classifiers as in Proposition \ref{prop:ppv-as-cond-exp}.  
    Then for any $g \in \G$,
    \begin{align*}
    \FPR_{\hard,g} &= \Pr[\hard(X_g) = \pT] \cdot \frac{1 - \E[\soft(X_g) \mid \hard(X_g) = \pT]}{1 - \E[\soft(X_g)]} \\
    \FNR_{\hard,g} &= \Pr[\hard(X_g) = \pF] \cdot  \frac{\E[\soft(X_g) \mid \hard(X_g) = \pF]}{\E[\soft(X_g)]}
    \end{align*}
    Assume that $Pr[\ty(X_g) = 1] > 0$ and $Pr[\ty(X_g) = 0] > 0$ (that is, assume $0 < \BR_g < 1$) so that $\FPR$ and
    $\FNR$ are well-defined.
\end{proposition}
\begin{proof}[Proof of Proposition~\ref{prop:fpr-as-cond-exp}]
We give the proof for $\FPR$, and the proof for $\FNR$ is similar. By applying Bayes' rule, we can write
\begin{align}
\FPR_{\hard, g} &= \Pr[\hard(X_g) = 1 \mid \ty(X_g) = 0] \nonumber \\
&= \Pr[\ty(X_g) = 0 \mid \hard(X_g) = 1] \cdot \frac{\Pr[\hard(X_g) = 1]}{\Pr[\ty(X_g) = 0]} \label{eq:fpr-as-cond-exp-1}
\end{align}
Noting that $\Pr[\ty(X_g) = 0 \mid \hard(X_g) = 1] = 1 - \PPV_{\hard, g}$, we can apply Proposition \ref{prop:ppv-as-cond-exp} and rearrange to write the RHS of Equation~\ref{eq:fpr-as-cond-exp-1} as follows.
\begin{equation}
\label{eq:fpr-as-cond-exp-2}
\text{RHS of (\ref{eq:fpr-as-cond-exp-1})} = \Pr[\hard(X_g) = 1] \cdot \frac{1 - \E[\soft(X_g) \mid \hard(X_g) = \pT]}{\Pr[\ty(X_g) = 0]}
\end{equation}
We note that $\Pr[\ty(X_g) = 0] = 1 - \E[\soft(X_g)]$ (Proposition~\ref{prop:baserate-calibrated}). Substituting this in to the RHS of Equation~\ref{eq:fpr-as-cond-exp-2}, we conclude the result.
\end{proof}

\subsection{General impossibility of equalizing PPV, NPV}

It is not always possible to directly post-process a soft groupwise calibrated
classifier into a hard one with equalized PPV (or NPV) for all groups, as we demonstrate by counterexample in
Proposition~\ref{prop:info-theory-impossibility-3}. Before proceeding, we note that our counterexample is somewhat
contrived---in particular, the AP induced by the soft classifier $\soft$ in the proof of
Proposition~\ref{prop:info-theory-impossibility-3} takes only one value on each group. When the AP of $\soft$ is more nicely structured on each group, we will see that there are general methods to equalize PPV (or NPV).

\begin{proposition}
\label{prop:info-theory-impossibility-3}
Fix two disjoint groups $g_1$ and $g_2$ with respective base rates $\BR_1$ and $\BR_2$ such that $\BR_1 \neq \BR_2$.
Then there exists a soft classifier $\soft$ that is groupwise calibrated, but for which there is no
post-processor  $\post: [0,1] \times \G \to \zo$ such that $\post \circ \soft$ equalizes PPV, unless $\Pr[\post(\BR_i, g_i) = 1] = 0$ for $i=1$ or 2.
\end{proposition}
\begin{proof}[Proof of Proposition \ref{prop:info-theory-impossibility-3}]
Consider the classifier $\soft$ such that $\soft(x) = \BR_1$ if $x \in g_1$
and $\soft(x) = \BR_2$ if $x \in g_2$.
This classifier is trivially groupwise calibrated. Since $\Pr[\post(\BR_i, g_i) = 1] > 0$ for $i=1$ and 2, we conclude that $\PPV_{\hard, g_i}$ is well-defined for $g_1$ and $g_2$.
The proof now follows from the characterization of PPV in Proposition~\ref{prop:ppv-as-cond-exp}.
This is because $\PPV_{\hard, g_i}$ is equal to the expectation of $\soft(X)$ where $X$ is drawn from a distribution with support contained in $g_i$, and hence it is equal to $\BR_{i}$, and $\BR_{1} \neq \BR_{2}$.
\end{proof}
The analogous statement regarding impossibility of equalizing NPV is formulated as Proposition~\ref{prop:info-theory-impossibility-npv} in Appendix~\ref{sec:appendix-npv}.

\subsection{A niceness Condition for APs}

We now give a non-degeneracy condition condition on APs motivated by the impossibility result for post-processing given by Proposition~\ref{prop:info-theory-impossibility-3}.
\begin{definition}[Niceness of APs]
\label{def:nice-docs}
Let $\G$ be a set of groups.  A distribution on calibrated scores $\pdf$ is \emph{nice} if $\supp(\pdf_g)$ is the same for all $g \in \G$.
\end{definition}
Note that this  condition rules out the counterexample given by Proposition~\ref{prop:info-theory-impossibility-3}, since the APs in the counterexample had different (in fact, disjoint) supports for different groups. Hence, we can hope to successfully post-process soft classifiers with nice APs.

\subsection{Equalizing PPV or NPV by Thresholding}
\label{sec:thresholds}
We pay special attention to thresholds because they are simple to understand and therefore very widely used.  We use
one slight modification to deterministic thresholds that adds an element of randomness: if a score is \emph{at} the
threshold, we randomly determine which side of the threshold it falls on, according to a distribution defined below.

\begin{definition}[Threshold Post-Processor]\label{defn:threshold-pp}
    A threshold post-processor $\post_{(\thresh, \threshRand)} : [0,1] \times \G \rightarrow \{1,0\}$ is a function
    from a score $s \in [0,1]$ and a group $g \in \G$, parameterized by $\tau$ and $\threshRand$.  The threshold
    parameter $\tau : \G \rightarrow [0,1]$ specifies the threshold for the group $g$, and $\threshRand : \G
    \rightarrow [0,1]$ is the probability of returning 1 when the input score $s$ is on the threshold $\tau(g)$.  It
    returns the following outputs:
\begin{align*}
    \post_{(\thresh, \threshRand)}(\score, g) = \begin{cases}
        \pT & \score > \tau(g) \\
        \pF & \score < \tau(g) \\
        \pT \text{ w.p. } \threshRand(g) \text{ else } \pF & \score = \tau(g)
    \end{cases}
\end{align*}
\end{definition}
In the setting of an infinite number of scores and a continuous domain (i.e. scores are represented by a
probability density function instead of a probability mass function), we can use purely deterministic threshold
functions in which $\threshRand \equiv 1$, and achieve very similar results for the rest of this section.

If both $\thresh(g)$ and $\threshRand(g)$ do not vary across groups $g \in \G$, then the post-processor is the same
across groups. In this case, we will call the post-processor a \emph{group blind} threshold post-processor, and will
overload $\thresh$ and $\threshRand$ to be constants.

We now study the effectiveness of thresholds for post-processing soft classifiers with nice APs. The main takeaways are:
\begin{enumerate}
\item If the APs are nice, then threshold post-processors can equalize PPV
(Propositions~\ref{prop:trivial-threshold-equalizes} and~\ref{prop:2-thresh-equal-ppv}).
\item However, group blind threshold post-processors are rather limited in their ability to equalize PPV (Proposition~\ref{prop:single-threshold-cannot-equalize}).
\item Furthermore, equalizing PPV with thresholds (group blind or otherwise) may have undesirable social consequences.
\item Thresholds cannot always equalize PPV and NPV simultaneously, even for nice APs (Proposition~\ref{prop:ppv-npv-impossibility}).
\end{enumerate}
Results 1-3 also apply to NPV (see Proposition~\ref{prop:single-threshold-cannot-equalize-npv}).

\subsubsection{Group Blind Thresholds}
We begin by classifying which group-blind threshold post-processors can equalize PPVs across all groups (Propositions~\ref{prop:trivial-threshold-equalizes} and~\ref{prop:single-threshold-cannot-equalize}). By symmetry, our arguments give a similar characterization for equalizing NPVs.

\begin{proposition}\label{prop:trivial-threshold-equalizes}
  For every nice groupwise calibrated soft classifier $\soft$ and for every group-blind threshold post-processor
  $\post_{(\thresh, \threshRand)}$ such that $\thresh(g) = \max(\supp(\pdf_g))$ 
  for all $g$, then the composed classifier $\hard = \post_{(\thresh, \threshRand)} \circ \soft$
  equalizes PPVs across all groups for which $\hard$ is non-trivial. 
  \end{proposition}

The existence of the threshold post-processors in Proposition~\ref{prop:trivial-threshold-equalizes} follows from the assumed finiteness of the range of the soft classifier. In the case where the range of the soft classifier is infinite, such post-processors may not exist.

\begin{proof}[Proof of Proposition~\ref{prop:trivial-threshold-equalizes}]
Any of the given post-processors only ever maps the largest score in the support of $\pdf_g$ to 1, for all groups $g$. Hence,
$\PPV_g$ is exactly the largest score in $\supp(\pdf_g)$. By the assumption that $\pdf$ is nice, $\supp(\pdf_g)$ is the same for all groups $g$, and hence the PPV is equalized across groups.
\end{proof}

We prove the analogous statement for NPV in Proposition~\ref{prop:trivial-threshold-equalizes-npv} in the Appendix.
We proceed to show that the post-processors described in Proposition~\ref{prop:trivial-threshold-equalizes} are the \emph{only} non-trivial, group blind post-processors that equalize PPV across groups in general, as we prove in Proposition~\ref{prop:single-threshold-cannot-equalize}.
\begin{proposition}\label{prop:single-threshold-cannot-equalize}
There exists a groupwise-calibrated soft classifier with a nice AP for which no non-trivial group blind threshold post-processor, other than the ones in Proposition~\ref{prop:trivial-threshold-equalizes}, can equalize PPV across groups.
\end{proposition}

 At a high level, the proof of Proposition~\ref{prop:single-threshold-cannot-equalize} works as follows: We can make the AP on one group uniform, and the AP of another group strictly increasing. Then, threshold post-processors naturally favor the latter group, as the AP for that group gives more weight to higher scores than lower ones when compared to the former AP. Our characterization of PPV (Proposition~\ref{prop:ppv-as-cond-exp}) features prominently in the proof.

In preparation to proving Proposition \ref{prop:single-threshold-cannot-equalize}, we first prove the following lemma:
\begin{lemma}
\label{lem:ppv-thresh}
Let $g_1, g_2 \in \G$ be two different groups, and fix a group-blind threshold post-processor $\post_{(\thresh, \threshRand)}$. Let $\pdf_{g_1, (\thresh, \threshRand)}$ be the expected conditional AP on scores $\geq \thresh$ that results from starting with the AP $\pdf_{g_1}$ over scores in group $g_1$ and conditioning on the scores that $\post_{(\thresh, \threshRand)}$ sends to 1, and similarly let $\pdf_{g_2, (\thresh, \threshRand)}$ denote the same type of conditional AP when starting with the $\pdf_{g_2}$ over scores in group $g_2$. \\
If $\pdf_{g_2, (\thresh, \threshRand)}$ strictly stochastically dominates $\pdf_{g_1, (\thresh, \threshRand)}$, then
\[ \PPV_{\post_{(\thresh, \threshRand)} \circ \soft, g_1} < \PPV_{\post_{(\thresh, \threshRand)} \circ \soft, g_2} \]
\end{lemma}
\begin{proof}
We use the characterization of PPV given in Proposition~\ref{prop:ppv-as-cond-exp}, for the special case where the post-processor thresholds as described above. We can write the PPV for group $g_1$ as follows:
\begin{align} \PPV_{\post_{(\thresh, \threshRand)} \circ \soft, g_1} &= \E_{X_1\sam \X_{g_1}}[\soft(X_1) \mid \post_{(\thresh, \threshRand)} \circ \soft(X_1) = 1] \nonumber \\
&=  \E_{\score \sam \pdf_{g_1, (\thresh, \threshRand)}}[\score] \label{eq:ppv-thresh-1}
\end{align}
where the second line follows from the definition of $\pdf_{g_1, (\thresh, \threshRand)}$.

Similarly, we have that
\begin{equation} \label{eq:ppv-thresh-2}
\PPV_{\post_{(\thresh, \threshRand)} \circ \soft, g_2} =  \E_{\score \sam \pdf_{g_2, (\thresh, \threshRand)}}[\score]
\end{equation}

Since $\pdf_{g_2, (\thresh, \threshRand)}$ stochastically dominates $\pdf_{g_1, (\thresh, \threshRand)}$, the expectation on the RHS of Equation~\ref{eq:ppv-thresh-2} is larger than the expectation on the RHS of Equation~\ref{eq:ppv-thresh-1}, yielding the result.
\end{proof}

\begin{proof}[Proof of Proposition \ref{prop:single-threshold-cannot-equalize}]

     Fix two groups $g_1$ and $g_2$ and a finite set of points $S \subset [0,1]$ such that the PMFs of the soft
     classifier $\soft$ on $g_1$ and $g_2$ have support equal to $S$ - that is, $\supp(\pdf_{g_1}) = \supp(\pdf_{g_2}) = S$. For concreteness, suppose that $|S| = 10$.

Let the PMFs of the soft classifier $\soft$ on these two groups respectively be given by $\pdf_{g_1}(\score) = 1/
|\supp(\pdf_{g_1})|$ and $\pdf_{g_2}(\score) \propto \score$ for all $\score \in S$, where $\pdf_{g_2}$ is normalized
with a constant such that it sums to 1. Fix a group blind threshold post-processor $\post_{(\thresh, \threshRand)}$
that is not one of the ones mentioned in Proposition~\ref{prop:trivial-threshold-equalizes}.  Since $\soft_{(\thresh,
\threshRand)}$ is group blind, its threshold function is a constant which we name~$\thresh$.

Let $\pdf_{g_1, (\thresh, \threshRand)}$ be the expected conditional AP on scores $\geq \thresh$ that results from starting with the AP $\pdf_{g_1}$ over scores in group $g_1$ and conditioning on the scores that $\post_{(\thresh, \threshRand)}$ sends to 1. We can get this conditional PMF by removing scores $\score < \thresh$, multiplying $\pdf_{g_1}(\thresh)$ by $\threshRand$, and re-normalizing the remaining values to get a distribution. Let $\pdf_{g_2, (\thresh, \threshRand)}$ be defined similarly.

We claim that $\pdf_{g_2, (\thresh, \threshRand)}$ strictly stochastically dominates $\pdf_{g_1, (\thresh, \threshRand)}$, which allows us to invoke Lemma~\ref{lem:ppv-thresh} to conclude that the PPV on the two groups are unequal. We now show that $\pdf_{g_2, (\thresh, \threshRand)}$ strictly stochastically dominates $\pdf_{g_1, (\thresh, \threshRand)}$. This is clearly true by design if $\threshRand = 0$ or $1$: in this case, the post-processor is simply a deterministic threshold function, and we know by design that $\pdf_{g_1, (\thresh, \threshRand)}$ is uniform while $\pdf_{g_2, (\thresh, \threshRand)}$ is a strictly increasing function. If $\threshRand = r$ for some $r \in (0,1)$, then we can write $\pdf_{g_1, (\thresh, \threshRand)}$ as a convex combination of $\pdf_{g_1, \thresh, 0}$ and $\pdf_{g_1, \thresh, 1}$ (with weight $r$ on the distribution where $\threshRand = 1$, and weight $1-r$ on the distribution where $\threshRand = 0$). We can write $\pdf_{g_2, (\thresh, \threshRand)}$ as the same convex combination of the conditional distributions over $g_2$ where $\threshRand = 0$ and $\threshRand = 1$. Since we already established stochastic domination for the cases where $\threshRand = 0$ and $\threshRand = 1$, this establishes stochastic domination for the case where $\threshRand \in (0,1)$.
\end{proof}

We achieve the same result for NPV in Proposition~\ref{prop:single-threshold-cannot-equalize-npv}.  In the setting
where the range of the soft classifier is infinite and continuous, we show in
Proposition~\ref{prop:single-threshold-cannot-equalize-cont} that a similar negative result holds, but without the
existence of the classifiers in Proposition~\ref{prop:trivial-threshold-equalizes}.

Propositions \ref{prop:trivial-threshold-equalizes} and \ref{prop:single-threshold-cannot-equalize} demonstrate the limitations of group blind thresholds on calibrated scores.  Though this method of post-processing has social appeal, it does not actually preserve the fairness properties that one would expect.  In the next section we repeat our analysis but relax our group blindness requirement.

\subsubsection{Group-Aware Thresholds}
If we allow the different groups to have different thresholds, then we grant
ourselves more degrees of freedom to be able to satisfy binary fairness
constraints. In particular, we can equalize PPV across groups in a more meaningful way than done in Proposition~\ref{prop:trivial-threshold-equalizes}.

Recall that the group blind threshold post-processors in Proposition~\ref{prop:trivial-threshold-equalizes} are the only
group blind threshold post-processors that work on certain nice APs (shown in
Proposition~\ref{prop:single-threshold-cannot-equalize}). However, these post-processors have the property that the
only score they map to $\pT$ is the largest score in the support, which can be undesirable for many applications.

In particular, all classifiers in Proposition~\ref{prop:single-threshold-cannot-equalize} make the PPV on each group
$g_i$
equal to the maximum score in the support of $\pdf_{g_i}$. However, the (not-necessarily-group-blind) threshold
post-processors in Proposition~\ref{prop:2-thresh-equal-ppv} below can make the PPV on each group equal to any fixed
value between the maximum base rate of $g_i$ and the maximum score in $\supp(\pdf_{g_i})$.

\begin{proposition}\label{prop:2-thresh-equal-ppv}
    Let $\G$ be a set of groups.  For any soft classifier $\soft$ with a nice AP $\pdf$ such that $\soft$ is
    groupwise calibrated over $\G$ and $|\supp(\pdf_g)| \geq 2$ for all $g \in \G$, then
    there exists a non group blind, non-trivial threshold post-processor $\post_{(\thresh, \threshRand)}$ that is not one of the ones from
    Proposition~\ref{prop:trivial-threshold-equalizes} such that the hard classifier
    $\hard = \post_{(\thresh, \threshRand)} \circ \soft$ equalizes PPV across $\G$.

    This holds even if we require that the PPV of all the groups is equal to an arbitrary value in $(\max_i \BR_{g_i},
    \score_{max})$, where $\max_i \BR_{g_i}$ is the maximum base rate among the groups $g_i \in \G$ and $\score_{max}$ is the
    maximum score in the support of $\pdf_{g_i}$.\footnote{For the case where the support of $\pdf_{g_i}$ is infinite, $\score_{max}$ should be the supremum of scores.}

    Moreover, since this post-processor is not group blind, it is not one of the post-processors described in
    Proposition \ref{prop:trivial-threshold-equalizes}.
\end{proposition}

In preparation to proving Proposition \ref{prop:2-thresh-equal-ppv}, we first prove the following claim:

\begin{claim}[Monotonicity of PPV and NPV]
\label{claim:ppv-monotonicity}
\label{claim:npv-monotonicity}
Fix a soft classifier $\soft$ and corresponding AP $\pdf$, as well as a group $g$. Fix group blind threshold post-processors $\post_{\thresh_1, \threshRand_1}$ and $\post_{\thresh_2, \threshRand_2}$ such that either $\thresh_1 < \thresh_2$ or $\thresh_1 = \thresh_2$ and $\threshRand_1 \geq \threshRand_2$. Then:\\
(a) $\PPV_{\post_{\thresh_1, \threshRand_1} \circ \soft, g} \leq \PPV_{\post_{\thresh_2, \threshRand_2} \circ \soft, g}$\\
(b) $\NPV_{\post_{\thresh_1, \threshRand_1} \circ \soft, g} \leq \NPV_{\post_{\thresh_2, \threshRand_2} \circ \soft, g}$
\end{claim}
\begin{proof}
We show conclusion (a); conclusion (b) is shown analogously. Define $\pdf_{g, \thresh_1, \threshRand_1}$ to be the conditional PMF on scores $\geq \thresh_1$ that results from starting with the AP $\pdf_{g}$ over scores in group $g$ and conditioning on the scores that $\post_{\thresh_1, \threshRand_1}$ sends to 1, and let $\pdf_{g, \thresh_2, \threshRand_2}$ be defined similarly (but for the threshold post-processor $\post_{\thresh_2, \threshRand_2}$).

We claim that $\pdf_{g, \thresh_2, \threshRand_2}$ stochastically dominates $\pdf_{g, \thresh_1, \threshRand_1}$, which yields the desired result by the characterization of PPV given in Proposition~\ref{prop:ppv-as-cond-exp} (and more explicitly written in Equations~\ref{eq:ppv-thresh-1} and~\ref{eq:ppv-thresh-2}).
\end{proof}

\begin{proof}[Proof of Proposition \ref{prop:2-thresh-equal-ppv}]
Fix a soft classifier $\soft$ with a nice AP $\pdf$ that is group-wise calibrated over $g_1, \ldots, g_n$, and fix a desired value $v \in (\max_i \BR_{g_i}, \score_{max})$. We will show that we can design a threshold post-processor $(\thresh, \threshRand)$ such that $\PPV_{g, \post_{(\thresh, \threshRand)} \circ \soft} = v$ for all groups $g$.

Fix an arbitrary group $g_j$. We proceed via a continuity argument to show that we can tune the threshold on $g_j$ to achieve PPV equal to $v$. The maximum possible value for $\PPV_{g_j, \post_{(\thresh, \threshRand)}}$ is $\score_{max}$ (achieved when $\thresh = \score_{max}$, by Claim~\ref{claim:ppv-monotonicity}), where $\score_{max}$ is the largest score in the support, as defined in the proposition statement\footnote{We ignore the trivial post-processor that never maps anything to 1, and hence leaves the PPV undefined.}.

Furthermore, note that, for any group, a lower bound on the PPV of a hard classifier on that group is the base rate of the group, where the lower bound is matched by the trivial post-processor that sends every score to 1. This follows from Claim~\ref{claim:ppv-monotonicity}.

We now claim that there is a setting of $\thresh(g_j)$ and $\threshRand(g_j)$ that achieves $\PPV_{g, \post_{(\thresh, \threshRand)} \circ \soft} = v$. We accomplish this by showing that there is a way to change
$(\thresh(g_j), \threshRand(g_j))$ such that the PPV decreases continuously. We first show:

\begin{claim}[Continuity of PPV]
\label{claim:ppv-continuity}
Fix a soft classifier $\soft$ and corresponding AP $\pdf$, as well as a group $g$. Suppose we have two post-processing algorithms, $\post_1$ and $\post_2$. Let $\pdf_{g, \post_1}$ be the expected conditional AP that results from starting with the AP $\pdf_{g}$ over scores in group $g$ and conditioning on the scores that $\post_1$ sends to 1, and define $\pdf_{g, \post_2}$ similarly. If $d_{TV}(\pdf_{g, \post_1}, \pdf_{g, \post_2}) < \eps$, then $|\PPV_{g, \post_1 \circ \soft} - \PPV_{g, \post_2 \circ \soft}| < O(\eps)$. Or in words, if the distance between the conditional APs is small, then the difference in PPV is small.
\end{claim}
\begin{proof}
Recall the characterization of PPV given in Proposition~\ref{prop:ppv-as-cond-exp} (and more explicitly written in Equation~\ref{eq:ppv-thresh-1}). This tells us that the PPV of group $g$ for the classifier $\post_1 \circ \soft$ is exactly the expectation of a random variable distributed according to $\pdf_{g, \post_1}$. Similarly, the PPV of group $g$ for the classifier $\post_2 \circ \soft$ is the expectation of a r.v. distributed according to $\pdf_{g, \post_2}$. Since both $\pdf_{g, \post_1}$ and $\pdf_{g, \post_2}$ have support bounded between 0 and 1, their expectations can differ by at most $\eps$, from which the claim follows. For completeness, we prove this below.

Suppose wlog that $\pdf_{g, \post_1}$ has the larger expectation. Let $S = \{ \score \in \supp(\pdf_g): \pdf_{g, \post_1}(\score) > \pdf_{g, \post_2}(\score)\}$. Then:
\begin{align*}
\PPV_{g, \post_1 \circ \soft} &= \sum_{\score \in \supp(\pdf_g)} \score \pdf_{g, \post_1}(\score) \\
&= \PPV_{g, \post_2 \circ \soft} + \sum_{\score \in S} \score (\pdf_{g, \post_1}(\score) - \pdf_{g, \post_2}(\score)) \\
&< \PPV_{g, \post_2 \circ \soft} + \eps
\end{align*}
where in the second line we use the fact that $\PPV_{g, \post_2 \circ \soft}$ is the expectation of $\pdf_{g, \post_2}$, and in the last line we use the fact that $\score \in [0,1]$ and that the TV-distance between the two distributions is less than $\eps$.
\end{proof}

Now, consider the following way to change $(\thresh(g_j), \threshRand(g_j))$. Fix $\eps > 0$, and an initial setting
for $(\thresh(g_j), \threshRand(g_j))$ s.t. $\thresh(g_j)$ is not the smallest item in the support or $\threshRand(g_j) > \eps$. Reduce $\threshRand(g_j)$ by $\eps$, wrapping around on the interval $(0,1]$ and decreasing $\thresh(g_j)$ to the next largest item in the support when this would otherwise make $\threshRand(g_j)$ negative.

This very minor transformation to the threshold changes the AP conditional on outputting 1 very slightly - so slightly that the TV distance between the old conditional AP and the new AP is at most some $\eps'$ which is a function of $\eps$. This lets us apply Claim~\ref{claim:ppv-continuity} to show that the PPV changes by at most a function of $\eps$. So as we take $\eps$ going towards 0, this shows that the PPV changes by an amount going towards 0. This establishes that the PPV changes ``continuously'' with respect to this deforming procedure.

By Claim~\ref{claim:ppv-monotonicity}, we have that the above deforming procedure can only decrease the PPV. Therefore, we can continuously decrease the PPV, starting from $\score_{max}$, by continuously deforming the threshold post-processor with the method above. Note that $\score_{\max} > v > \max_i \BR_{g_i} \geq \BR_{g_j}$. By the Intermediate Value Theorem, there must be a setting of $(\thresh(g_j), \threshRand(g_j))$ such that $\PPV_{g_j, \post_{(\thresh, \threshRand)} \circ \soft} = v$.

\end{proof}

We assert the analogous statement
for the case of NPV in Claim~\ref{claim:npv-continuity}. The corresponding statement for the case of soft classifiers with infinite range is asserted in Proposition~\ref{prop:equalize-ppv-cont}.

\subsubsection{The Limitations of Thresholding}
\label{sec:limits-of-thresholds}
While Proposition~\ref{prop:2-thresh-equal-ppv} shows that a threshold post-processor can equalize the PPV across
$n$ groups, this threshold post-processor can be unsatisfying from a social justice standpoint. Consider an example
with two groups $g_1$ and $g_2$, where group $g_2$ is ``privileged'' by having a higher base rate of, say, credit
worthiness. Suppose that we have an AP that is decreasing with respect to score on group $g_1$, and increasing with
respect to score on group $g_2$. This is illustrated in Example \ref{ex:social-unsat} and
Figure~\ref{fig:undesirable-results}. This means that a group blind threshold post-processor yields larger PPV on $g_2$, since large scores are given more weight in $g_2$. So, to equalize the PPV between the two groups, we will classify more low scores as positive in $g_2$ than $g_1$. This effectively means that our threshold on group $g_2$ is \emph{more lenient} than our threshold on $g_1$, which seems blatantly unfair, since $g_2$ is the privileged group in the first place! 

\begin{example}[Socially Unsatisfying Example]
\label{ex:social-unsat}
Fix groups $g_1$ and $g_2$, and we fix the AP of the soft classifier $\soft$ as follows. Let $\supp(\pdf_{g_i})$ be
$\{0, 0.01, 0.02, \ldots, 1)$ for $i=1$ and $2$, let $\pdf_{g_1}(\score) \propto a-\score$ for appropriately selected
constant $a> 0$ and let $\pdf_{g_2}(\score) \propto \score$.  Group $g_2$ has a higher base rate and may have social
advantages over group $g_1$.

Let $\post_{(\thresh, \threshRand)}$ be a non-trivial post-processor.
If $\post_{(\thresh, \threshRand)}$ were group blind, then by Lemma~\ref{lem:ppv-thresh}, since
$\pdf_{g_2, \post_{(\thresh, \threshRand)}}$ stochastically dominates $\pdf_{g_1, \post_{(\thresh, \threshRand)}}$,
the PPV on $g_2$ must be larger than the PPV of $g_1$.

To equalize PPV, by Claim~\ref{claim:ppv-monotonicity}, we must have either $\thresh(g_2) < \thresh(g_1)$, or $\thresh(g_2) =
\thresh(g_1)$ and $\threshRand(g_2) < \threshRand(g_1)$.
The disadvantaged group is now held to a \emph{higher} standard than the privileged group to maintain equality of PPV.
Figure~\ref{fig:undesirable-results} illustrates example thresholds that equalize the PPV.
\end{example}
\begin{figure}
    \centering
    \includegraphics[width=0.5\hsize]{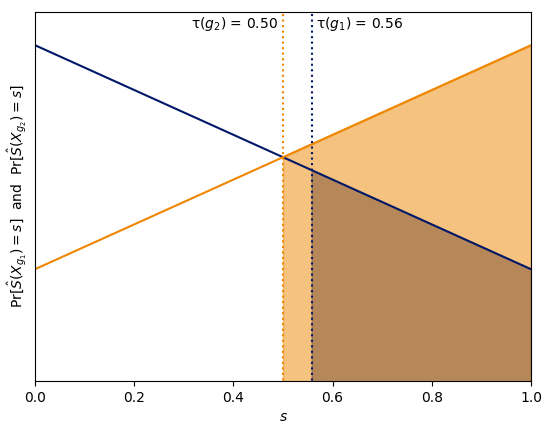}
    \caption{Accompanying Example~\ref{ex:social-unsat}, the PPV for both groups is 0.77.  However, the threshold for
        $g_1$ (dark blue) is \emph{higher} than
    the threshold for $g_2$ (orange), even though $g_2$ is likely the more privileged group.}
    \label{fig:undesirable-results}
\end{figure}

The only property needed for Example~\ref{ex:social-unsat} is that the AP of one (privileged) group stochastically dominates the AP of another (unprivileged) group.
We suspect that this to will occur in many settings. Indeed, the COMPAS scores we analyze in Section~\ref{sec:experiments} have this property, when considering Caucasians as the privileged group and African-Americans the  unprivileged group. As in the above example, using group-aware thresholds to equalize PPV between the groups on COMPAS data results in a more permissive decision rule for Caucasians, demonstrating a problem with this approach (see also Figure~\ref{fig:compas-thresholds}).

Furthermore, thresholding cannot in general equalize both PPV and NPV simultaneously, even for nice APs and using non-group blind thresholds.

\begin{proposition}
\label{prop:ppv-npv-impossibility}
Fix groups $g_1$ and $g_2$. There exists a soft classifier $\soft$ with a nice AP~$\pdf$ such that no threshold post-processor can simultaneously equalize PPV and NPV between groups $g_1$ and $g_2$.
\end{proposition}

Before proving the statement, we first show that the base rates of $g_1$ and $g_2$ can be written as convex combinations of the $\PPV$ and $1 - \NPV$ on the respective groups:

\begin{claim}
\label{claim:baserate-convex-comb}
Fix any group $g \in \G$, and let the hard classifier $\hard$ be non-trivial.  Then the base rate of $g$ can be written as a convex combination of $\PPV_{g, \hard}$ and $1 - \NPV_{g, \hard}$
\end{claim}
\begin{proof}
\begin{align*}
\BR_g &= \Pr[Y(X_g) = 1] \\
&= \Pr[Y(X_g) = 1 \mid \hard(X_g) = 1] \Pr[\hard(X_g) = 1] \\
& \quad + \Pr[Y(X_g) = 1 \mid \hard(X_g) = 0] \Pr[\hard(X_g) = 0] \\
&= \PPV_{g, \hard} \cdot \theta + (1 - \NPV_{g, \hard}) \cdot (1 - \theta)
\end{align*}
where $\theta := \Pr[\hard(X_g) = 1]$.
\end{proof}

A simple intuition for the proof of Proposition~\ref{prop:ppv-npv-impossibility} is as follows. Suppose we have two
groups, and the soft classifier is almost perfect on one group - for all but a small fraction of people, it gives the
correct binary score, and gives the remaining people score 0.5. On the other group, it is the opposite - almost every
person is given score 0.5, and there are only a few people given their ground truth score. The corresponding APs of
each group have equal supports, and therefore are ``nice.'' However, it is clear that any threshold post-processor on the first group will have extremely high PPV and NPV, while any threshold post-processor on the second group will have to make decision on where to round the people in the 0.5 bucket, and will correspondingly either have low PPV or NPV.

This example is somewhat unsatisfying, as the AP satisfies niceness by a technicality (for example, it is extremely close to a AP with disjoint supports
between groups). The proof generalizes the above example to a case where scores for one group are slightly more
``correlated'' with the ground truth labels than scores for the other group.
Appendix~\ref{sec:appendix-cont} shows the corresponding result when the range of the soft classifier is infinite.

\begin{proof}[Proof of Proposition~\ref{prop:ppv-npv-impossibility}]
Let $\supp(\pdf_{g_i})$ be $\{ 0.1, 0.2, \ldots, 0.9\}$ for $i=1$ and $2$. Let $\pdf_{g_1}$ be uniform over scores, and let $\pdf_{g_2}(\score) = -a(\score - 1/2)^2 + b$ for some appropriately chosen constants $a, b \geq 0$ such that $\pdf_{g_2}$ is a valid probability distribution.

We claim that there is no threshold post-processor that can equalize PPV and NPV simultaneously for these two groups. First, note that the base rates for the two groups are equal to $1/2$ by design, due to the symmetric nature of the AP on each group. Just like in the proof of Proposition~\ref{prop:single-threshold-cannot-equalize}, we use the notation $\pdf_{g_1, \geq (\thresh, \threshRand)}$ to denote the conditional AP supported on scores $\geq \thresh$ that results from starting with the AP $\pdf_{g_1}$ over all scores in group $g_1$ and conditioning on the scores that $\post_{(\thresh, \threshRand)}$ sends to 1, and define $\pdf_{g_2, \geq (\thresh, \threshRand)}$ similarly. Let $\pdf_{g_i, \leq (\thresh, \threshRand)}$ denote the conditional AP starting from group $g_i$ and conditioning on the scores that $\post_{(\thresh, \threshRand)}$ sends to 0.

We proceed by a case analysis on the location of the threshold for group $g_1$ - that is, on $\thresh(g_1)$.

\paragraph{\textbf{Case 1: $\thresh(g_1) \geq 1/2$.}}
First, suppose additionally that $\post_{(\thresh, \threshRand)}$ is group blind over $g_1, g_2$, then $\pdf_{g_2, (\thresh, \threshRand)}$ is strictly stochastically dominated by $\pdf_{g_1, (\thresh, \threshRand)}$ so the PPV is lower on $g_2$ than $g_1$ (Lemma~\ref{lem:ppv-thresh}). Therefore, to equalize the PPV between the two groups, the threshold on group $g_2$ must be to the ``right'' of the threshold on $g_1$ (that is, either $\thresh(g_2) > \thresh(g_1)$ or $\thresh(g_2) = \thresh(g_1)$ and $\threshRand(g_2) < \threshRand(g_1)$), which follows from Claim~\ref{claim:ppv-monotonicity}.

However, we claim that this setting of thresholds makes the NPV on group $g_2$ lower than the NPV on group $g_1$.
It suffices to show that, for any constant threshold post-processor classifier $\post_{(\thresh, \threshRand)}$ with $\thresh(g_1) = \thresh(g_2) \geq 1/2$, the NPV on group $g_2$ is lower than the NPV on group $g_1$. This implies that the same statement holds when either $\thresh(g_2) > \thresh(g_1)$ or $\thresh(g_2) = \thresh(g_1)$ and $\threshRand(g_2) < \threshRand(g_1)$ as well, due to the monotonicity property of NPV (Claim~\ref{claim:npv-monotonicity}).
This suffices to prove Proposition~\ref{prop:ppv-npv-impossibility} for the case when $\thresh(g_1) \geq 1/2$: we need to set $\thresh(g_2), \threshRand(g_2)$ such that either $\thresh(g_2) > \thresh(g_1)$ or $\thresh(g_2) = \thresh(g_1)$ and $\threshRand(g_2) < \threshRand(g_1)$ in order to equalize PPV, but this leaves the NPV on group $g_2$ lower than the NPV on group $g_1$.

Now we proceed to show that the NPV is lower on $g_2$ when the threshold post-processor is constant. Fix a constant threshold post-processor $\post_{(\thresh, \threshRand)}$ with $\thresh \geq 1/2$. We have already established that this means that $\PPV_{g_2, \post_{(\thresh, \threshRand)} \circ \soft} < \PPV_{g_1, \post_{(\thresh, \threshRand)} \circ \soft}$.

Now, since  the base rates of $g_1$ and $g_2$ are equal, this implies that
\[ 1 - \NPV_{g_2, \post_{(\thresh, \threshRand)} \circ \soft} > 1 - \NPV_{g_1, \post_{(\thresh, \threshRand)} \circ \soft}\]
Rearranging, this shows that $\NPV_{g_2, \post_{(\thresh, \threshRand)} \circ \soft} < \NPV_{g_1, \post_{(\thresh, \threshRand)} \circ \soft}$, finishing the proof of this case.

\paragraph{\textbf{Case 2: $\thresh(g_1) < 1/2$.}}
The argument in this case is symmetrical to the previous case, where we switch the roles of PPV and NPV in the argument. We sketch it for the sake of completeness.

Fix a constant threshold post-processor such that $\thresh(g_1) = \thresh(g_2) < 1/2$. By design, $\pdf_{g_2, \leq (\thresh, \threshRand)}$ strictly stochastically dominates $\pdf_{g_1, \leq (\thresh, \threshRand)}$. Hence, the $\NPV$ on group $g_2$ is smaller than the NPV on group $g_1$ (this follows from an analogous version of Lemma~\ref{lem:ppv-thresh} for NPV instead of PPV). This means that the threshold post-processor cannot be constant to equalize $\NPV$s - it must be moved such that either $\thresh(g_2) < \thresh(g_1)$ or $\thresh(g_2) = \thresh(g_1)$ and $\threshRand(g_2) > \threshRand(g_1)$ (Claim~\ref{claim:npv-monotonicity}).

However, by the same convex combination argument as in the previous case, we get that any such constant threshold post-processor must make the PPV on $g_2$ strictly smaller than the PPV on $g_1$. By the monotonicity of PPV, this means that any threshold post-processor with either $\thresh(g_2) < \thresh(g_1)$ or $\thresh(g_2) = \thresh(g_1)$ and $\threshRand(g_2) > \threshRand(g_1)$ must also make the PPV on $g_2$ smaller than the PPV on $g_1$, finishing the proof.
\end{proof}

\subsection{Equalizing APs}
\label{sec:docs-without-deferrals}

While thresholding is a conceptually simple approach to post-processing a soft classifier, its power is limited. We now consider a very different approach using
soft post-processors to equalize the APs across groups of a soft classifier. The intuition is that if the APs are equal across groups, then any hard post-processor that is \emph{group blind} should result in equal PPV, NPV, FPR, and FNR. We formalize this intuition in Claim~\ref{claim:equalizing-docs-good}.

Let $\soft$ be a soft classifier and for each group $g \in \G$, let $\pdf_g$ be the AP of $\soft$ for group $g$. For a soft post-processor $\postsoft$, let
$\soft' = \postsoft \circ \soft$ and let $\pdf'_g$ be the corresponding AP for group $g$.

Our goal is to find a soft post-processor $\postsoft$ such that $\soft'$ is groupwise calibrated, and $\pdf'_g = \pdf'_{g'}$ for all $g,g'\in\G$. In this section, we describe only one approach to constructing $\postsoft$ which we call \emph{mass averaging}.

The approach of equalizing APs has a fundamental weakness: if $\pdf'_g = \pdf'_{g'}$ and both are calibrated, then $\BR_g = \BR_{g'}$. This severely limits applicability of this approach.
However, this limitation will removed in Section~\ref{sec:docs-with-deferrals} by allowing deferrals.

\begin{claim}
  \label{claim:equalizing-docs-good}
  If the APs are equal for two groups, then PPV, NPV, FPR, and FNR are
  equalized by any hard post-processor $\decide$ satisfying group blindness.
\end{claim}

The group-blindness requirement in the claim is necessary: consider the (not group blind) post-processor that outputs 0 on one group and 1 on the other;
PPV will not be equalized.

\begin{proof}[Proof of Claim:~\ref{claim:equalizing-docs-good}]
  We prove only that PPV is equalized; the remaining properties may be proved similarly.
  Let $\hard = \decide \circ \soft$ be a hard classifier $\hard$ that is a group blind post-processor $\post$ composed with a calibrated soft
  classifier $\soft$. All probabilities below are over $X_g\sam \X_g$, and the coins of $\soft$ and $\decide$.
  \begin{align*}
     &\PPV_{\hard,g}
     = \Pr[\ty(X_g) = 1 \mid \hard(X_g,g)=1] \\
     &= \frac{\Pr[\ty(X) = 1 ]}{\Pr[\hard(X_g,g) =1]} \\
     & \quad \cdot
     \sum_{\score \in \Supp(\pdf_g)} \biggl( \frac{\Pr[\ty(X_g) = 1 \mid \soft(X_g) = \score]\cdot\Pr[\soft(X_g) =
         \score ]}{\Pr[\ty(X_g) = 1]}\\
     &\quad\quad\quad\quad\quad \cdot \Pr[\hard(X_g,g) = 1 \mid \soft(X_g) = \score, \ty(X_g) = 1]\biggr)
  \end{align*}
  Each factor in this product is equal across groups by the assumptions. Namely, $\Pr[\ty(X_g) = 1]$ is equalized by calibration and equalized APs;
  $\Pr[\hard(X_g,g) =1 ]$ by group blindness and equalized APs;
  $\Pr[\hard(X_g, g) = 1\mid \soft(X_g) = \score, \ty(X_g) = 1]$ by group blindness;
  $\Pr[\ty(X_g) = 1\mid \soft(X_g) = \score]$ by calibration; and finally
  $\Pr[\soft(X_g) = \score ]$ by equalized APs.
\end{proof}

\subsubsection{Mass Averaging}
The mass-averaging technique is best illustrated with an example. Suppose that $\pdf_{g_1}$ is uniform over $\{0,0.5,1\}$, and $\pdf_{g_2}$ is uniform over $\{0,1\}$. It is easy to define a soft post-processor \postsoft which equalizes these two APs. On $g_1$, we leave the score unchanged: $\postsoft(\score,g_1) = \score$. On $g_2$, we compute the output as
$$
\postsoft(\score,g_2) = \begin{cases}
\score & \text{w.p. } 2/3 \\
0.5 & \text{w.p. } 1/3
\end{cases}
.$$
The APs for groups $g_1$ and $g_2$ of the resulting soft classifier $\soft' = \postsoft \circ \soft$ are equal, and are equal to $\pdf_{g_1}$.

In the example, the probability mass is being redistributed by averaging the scores. This can be equivalently viewed as adding noise to the scores and then recalibrating the scores, something discussed in \cite{Corbett}.

More generally, a mass-averaging post processor $\postsoft$ assigns to each possible pair $(\score,g)$ a distribution over possible output scores $\score'$. Such a \postsoft is fully specified by $k\cdot k' \cdot |\G|$
parameters, where $k$ is the number of possible values of $\score$ and $k'$ is the number of possible values of $\score'$. Given a soft classifier $\soft$ and a mass-averaging post processor $\postsoft$, the constraint that the resulting APs are equalized across groups is linear in these parameters. Such classifiers, therefore, may be found by a linear program. We do not explore the choice of mass-averaging post-processors further.



\section{Post-Processing Calibrated Classifiers with Deferrals}
\label{sec:deferrals}
In the first part of the paper, we considered the problem of post-processing calibrated soft classifiers, which output a score
$\score \in [0,1]$, into fair hard classifiers, which output a decision in $\py\in\{\pF,\pT\}$, subject to a number of group fairness conditions.
In the remainder of this work, we reconsider this problem, but with one important change: we allow classifiers to
``refuse to decide'' by outputting the special symbol $\punt$. We call such classifiers \emph{deferring} classifiers,
borrowing the nomenclature from \cite{MPZ17}. The output $\punt$ is the deferring classifier's way of refusing to make a decision and deferring to a downstream decision maker. For example, a risk assessment tool might aid a parole board to make a decision by categorizing an individual as high risk or low risk, or it might output $\punt$---providing no advice and deferring to the  judgment of the board.

We now modify our notation appropriately. Instances $x$ are still associated with a true type $\ty(x) \in \{\tF, \tT\}$ and a group $\group(x) \in \G$. A deferring hard classifier $\hard$ is a randomized function $\hard:\X \to \{\pF,\pT, \punt\}$.
A deferring soft classifier is a randomized function $\soft:\X\to[0,1]\cup \{\punt\}$.
A deferring hard (resp.~soft) post-processor is a randomized function $\post: [0,1]\cup\{\punt\} \times \G \to
\{\pF,\pT,\punt\}$ (resp.~$\postsoft: [0,1] \cup \{\punt\} \times \G \to [0,1] \cup \{\punt\}$) that takes as input the
output of a deferring soft and post-processes it into a deferring hard (resp.~soft) classifier.
We also introduce new versions of the FPR and FNR, conditioned on not deferring.  

\begin{definition}\label{defn:conditional-measures}
The \emph{conditional false positive rate}  and \emph{conditional false negative rate} of a deferring hard classifier $\hard$ for a group $g$ are, respectively: 
\begin{align*} \cFPR_{\hard,g} &= \Pr[\hard(X_g) = \pT \mid \ty(X_g) = \tF, \hard(X_g) \neq \punt] \\
\cFNR_{\hard,g} &= \Pr[\hard(X_g) = \pF \mid \ty(X_g) = \tT, \hard(X_g) \neq \punt].
\end{align*}
\end{definition}

We additionally consider a version of the accuracy profile conditioned on not deferring, which we call the \emph{conditional AP}. For non-deferring soft classifiers, Definitions~\ref{defn:conditional-docs} and \ref{defn:docs} coincide.
\begin{definition}
  \label{defn:conditional-docs}
 The \emph{conditional AP} $\pdf_g$ of a classifier $\soft$ for a group $g$ is the PMF of $\soft(X_g)$, conditioned on not outputting $\punt$. That is, for $\score \in [0,1]$,
 $\pdf_g(\score) =
 \Pr[\soft(X_g) = \score \mid \soft(X_g) \neq \punt].$ Note that the conditional AP is undefined if $\Pr[\soft(X_g) \neq \punt] = 0$.

 Abusing notation, we denote by $\pdf$ the collection $\{\pdf_g\}_{g\in \G}$, and call it the \emph{conditional AP of~\soft}.
\end{definition}

The conditional error rates are
applicable generally, but they can be difficult to interpret.
The consequences of using the conditional FPR and FNR are discussed further in Section \ref{sec:discussion} along with a discussion of different deferral models. 
They are also amenable to the consideration of additional goals which we will briefly address. 
For example, one could seek to minimize the total deferral rate, equalize the deferral rate among groups, or prefer deferrals on positive instances.

\subsection{Thresholding with deferrals}
\label{sec:threshold-deferrals}
We return now to the problem of post-processing of calibrated soft classifiers, but now with the extra power of deferring on some inputs. We revisit the two approaches discussed in Section~\ref{sec:postprocess-limits}: thresholding and equalizing APs.

Proposition~\ref{prop:info-theory-impossibility-3} stated PPV and NPV cannot both be equalized across groups in general when using only a single threshold per group.
By using two thresholds per groups and deferring on some inputs, PPV and NPV can always be equalized across groups.

We post-process using two thresholds per group as follows:
return $\pF$ when $\score$ is lower than the first threshold, return $\pX$
between the thresholds, and return $\pT$ above the second threshold, as shown in Figure \ref{fig:two-thresh-per-group}.
This buys us more degrees of freedom when equalizing binary constraints, and it has the useful property
that we say $\pX$ on the instances where we are the least confident about the predicted type.

We adapt our notation as follows:
\begin{definition}[Deferring Threshold Post-Processor]\label{def:thresh-deferral}
A deferring threshold post-processor $\post_{(\tau_0,\tau_1,\threshRand_0,\threshRand_1)}$ assigns to each group $g$ two thresholds $\tau_0(g),\tau_1(g) \in \Supp(\pdf_g)$, and two probabilities $\threshRand_0(g),\threshRand_1(g) \in [0,1]$, with the following requirements:
\begin{enumerate}
    \item for all $g \in \G$, $\tau_0(g) \le \tau_1(g)$
    \item for all $g \in \G$ for which $\tau_0(g) = \tau_1(g)$, $\threshRand_1(g) + \threshRand_0(g) \le 1$.  This corresponds to the case
        where the two thresholds are the same, and therefore individuals with that score must be mapped to $\pT$ with
        probability $\threshRand_1(g)$, and to $\pF$ with probability $\threshRand_0(g)$, with the remainder mapped to
        $\pX$.
\end{enumerate}
The corresponding threshold post-processor is defined as follows:\\
\resizebox{\hsize}{!}{
    \begin{minipage}{\linewidth}
\begin{align*}
    \post_{(\tau_0, \tau_1, \threshRand_0, \threshRand_1)}(\score, g) = \begin{cases}
        \pT & \score > \tau_{1}(g) \\
        \pF & \score < \tau_0(g) \\
        \pX & \tau_0(g) < \score < \tau_1(g) \\
        \pT \text{ w.p. }\threshRand_1(g), \text{ else }\pX & \score = \tau_1(g) \\
        \pF \text{ w.p. }\threshRand_0(g), \text{ else }\pX & \score = \tau_0(g) \\
        \pT \text{ w.p. }\threshRand_1(g), \pF \text{ w.p. }\threshRand_0(g), \text{ else }\pX & \score = \tau_0(g) = \tau_1(g) \\
        \pX & \score = \pX
\end{cases}
\end{align*}
\end{minipage}}
\end{definition}

\begin{figure}
    \centering
    \includegraphics[width=0.5\hsize]{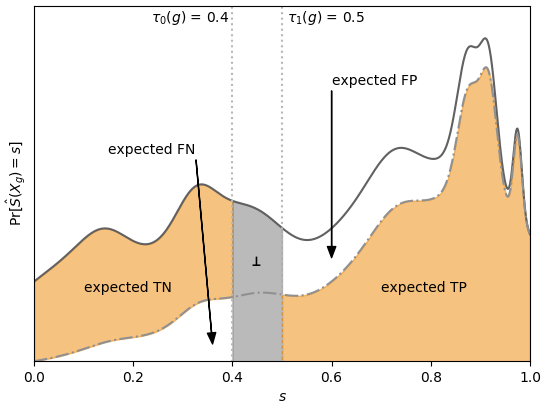}
    \caption{For threshold post-processors with deferrals, defer between the thresholds.}
    \label{fig:two-thresh-per-group}
\end{figure}

Using two thresholds allows the equalization of both PPV and NPV across groups in general, whereas without deferrals we
could only equalize one or the other.  We first demonstrate the existence of post-processors that are fairly limited, analogously to those defined in
Proposition~\ref{prop:trivial-threshold-equalizes}.

\begin{proposition}\label{prop:trivial-thresholds-deferrals}
    Let $\soft$ be a soft classifier with a nice AP for a set of groups $\G$.
    Then every threshold post-processor $\post_{(\thresh_0, \thresh_1, \threshRand_0, \threshRand_1)}$ satisfying the following properties 
    equalizes both the PPV and NPV for all groups in $\G$ of the composed classifier $\hard = \post_{(\thresh_0, \thresh_1, \threshRand_0, \threshRand_1)} \circ \soft$:
\begin{enumerate}
                \item $\tau_0(g) = \min(\supp(\pdf_g))$ for all $g$ 
                \item $\threshRand_0(g) > 0$ for all $g$.
                \item $\tau_1(g) = \max(\supp(\pdf_g))$ for all $g$ 
                \item $\threshRand_1(g) > 0$ for all $g$.
\end{enumerate}
\end{proposition}
Notice that these classifiers cannot be trivial, because we defined $\threshRand_0$ and $\threshRand_1$ in a way that prohibits the possibility that the
composed classifier never returns 0 or 1.  For the cases where the range of soft classifier outputs is infinite and there is no max or min element, these
classifiers do not exist.
\begin{proof}[Proof of Proposition~\ref{prop:trivial-thresholds-deferrals}]
    The reasoning is similar to the non-deferral case for equality of PPV alone.  The thresholds only allow one score
    $\score$ to map to $\pF$ and one score to map to $\pT$.  Thus, PPV for both groups is equal to the largest score
    in the support and NPV for both groups is equal to 1 minus the smallest score in the support.
\end{proof}

Now, much like in Proposition~\ref{prop:2-thresh-equal-ppv}, which showed the existence of meaningful non-trivial threshold post-processors that equalized PPV
across groups, we show the existence of meaningful, nontrivial \emph{deferring} threshold post-processors that equalize PPV and
NPV across groups.

\begin{proposition}\label{prop:2n-thresholds-deferrals}
  Let $\soft$ be a soft classifier with a nice AP that is groupwise calibrated for a set of groups $\G$. Suppose that
  $\vert\supp(\pdf_{g})\vert \ge 2$ for all $g \in \G$.  Then there exists a non-trivial threshold post-processor $\post_{(\thresh_0, \thresh_1, \threshRand_0, \threshRand_1)}$ that is not one of those defined 
  in Proposition~\ref{prop:trivial-thresholds-deferrals}, such that the hard classifier $\hard = \post_{(\thresh_0,
  \thresh_1, \threshRand_0, \threshRand_1)} \circ \soft$ equalizes $\PPV_g$ and $\NPV_g$ for all $g \in \G$. 
 \end{proposition}

The main idea of the proof of this proposition is
to use Proposition \ref{prop:2-thresh-equal-ppv} twice: once for getting thresholds to equalize the PPV, and once
for thresholds to equalize the NPV.  These thresholds may be invalid because there may be a group $g$ for which
$\tau_0(g) > \tau_1(g)$.  We use Claims \ref{claim:ppv-monotonicity} and \ref{claim:ppv-continuity} to allow ourselves
to push the PPV thresholds toward 1 and the NPV thresholds toward 0 until they no longer overlap, while still
maintaining equalization of PPV and NPV for the other groups.

\begin{proof}[Proof of Proposition \ref{prop:2n-thresholds-deferrals}]
Recall by Claim~\ref{claim:ppv-monotonicity} that the PPV of $\soft$ on a group $g$ monotonically increases as
$\tau_1(g)$ increases and, if $\thresh_1(g)$ is constant, as $\threshRand_1(g)$ increases.  By
Claim~\ref{claim:npv-monotonicity}, NPV monotonically increases as $\tau_0(g)$ decreases, and, if $\thresh_0(g)$ is
constant, as
$\threshRand_0(g)$ increases.  Recall by Claim \ref{claim:ppv-continuity} that if the total
variance distance between conditional APs (conditioned on being post-processed to a result of $\pT$) for different
groups is at most $\epsilon$, then the PPV difference for these groups is bounded by $O(\epsilon)$.  Thus, $\PPV_g$ is
continuous and monotonically increasing with regard to  $(\tau_1(g), \threshRand_1(g))$.  Similarly, $\NPV_g$ is continuous and monotonically
increasing with $(-\tau_0(g), \threshRand_0(g))$.

We know by Proposition \ref{prop:2-thresh-equal-ppv} that there exists a non-group blind threshold rule (without deferrals) that equalizes the PPV among the
groups.  By the analogous Proposition~\ref{prop:2-thresh-equal-npv}, there exists a (different) non-group-blind threshold rule that equalizes the NPV among the
groups.  For both of these, we know that they are not the classifiers from
Proposition~\ref{prop:trivial-thresholds-deferrals}.

If the thresholds meet the conditions of being a deferring post-processor listed in Definition~\ref{def:thresh-deferral}, then the statement is proven.  If they do
not meet the conditions because the thresholds ``overlap,'' we repeat the following procedure until the conditions are met:

\begin{enumerate}
    \item Let $g$ be a group for which the conditions are not met, i.e. either $\tau_0(g) > \tau_1(g)$, or 
          $\tau_0(g) = \tau_1(g)$ and $\threshRand_0(g) + \threshRand_1(g) > 1$.

    \item If $\tau_0(g) > \tau_1(g)$, define $t' = \frac{\thresh_0(g) + \thresh_1(g)}{2}$.  Let $t = \argmin_{s \in \supp(\pdf_g)} \vert s - t' \vert$.
        Notice that $t \le \tau_0(g)$ and $t \ge \tau_1(g)$, but because by assumption $\tau_0(g) > \tau_1(g)$, $t$ cannot be equal to both thresholds.  
        Set the new value for both thresholds to $t$: $\tau_0(g) = \tau_1(g) = t(g)$.
\item If $\threshRand_0(g) + \threshRand(g) > 1$, then do the following:
    \begin{enumerate}
        \item If $\thresh_0(g)$ remained unaltered in the previous step, then keep $\threshRand_0(g)$ the same, and set
        $\threshRand_1(g) = 1 - \threshRand_0(g)$.
    \item If $\thresh_1(g)$ remained unaltered in the previous step, then set $\threshRand_1(g)$ the same and set
        $\threshRand_0(g) = 1 - \threshRand_1(g)$.
    \item If neither of these is true, then let $r =
\frac{\threshRand_0(g)}{\threshRand_0(g) + \threshRand_1(g)}$.  Set $\threshRand_0(g) = r$ and $\threshRand_1(g) =
1-r$.
    \end{enumerate}

\item These thresholds are no longer overlapping, but they altered $\PPV_g$ and $\NPV_g$.
    Notice that, by the monotonicity properties described above, the threshold rules were changed in ways that can only increase $\PPV_g$ or $\NPV_g$:
\begin{enumerate}
    \item $\thresh_1(g)$ has increased or remained constant
    \item if $\thresh_1(g)$ remained constant, then $\threshRand_1(g)$ also remained constant
    \item $\thresh_0(g)$ has decreased or remained constant
    \item if $\thresh_0(g)$ remained constant, then $\threshRand_0(g)$ remained constant
\end{enumerate}
The new PPV and NPV for $g$ may now be higher than those of the other groups.

\item For all other groups $g' \ne g$, by the Intermediate Value Theorem and the continuity of NPV, there exists some $(\thresh_0(g'),
\threshRand(g'))$ that sets $\NPV_{g'} = \NPV_g$, and by the monotonicity of NPV, this threshold is lower than the old
one.  Similarly, there exists some $(\thresh_1(g'), \threshRand(g'))$ that sets $\PPV_{g'} = \PPV_g$ and it is higher than
the old one.
By the monotonicity of PPV and NPV, we know that this process will not cause non-overlapping thresholds to become overlapping.
\end{enumerate}

The ultimate effect of these steps was to reduce the number of overlapping thresholds by at least one.  We can
repeat this process up to $2|\G|$ times until none of the thresholds overlap.

Notice that this classifier is not one of the ones from Proposition~\ref{prop:trivial-thresholds-deferrals} - if we did not
have to correct for ``overlapping,'' then this is true by assumption, and if we did do the correction process, then
$\thresh_0(g) = \thresh_1(g)$ for at least one $g$, and by assumption we had $\vert\Supp(\pdf)\vert \ge 2$.

Thus, we have created valid a post-processor $\post_{(\thresh_0, \thresh_1, \threshRand_0, \threshRand_1)}$ that
equalizes PPV and NPV for all groups simultaneously and is not one of the ones in
Proposition~\ref{prop:trivial-thresholds-deferrals}, proving the claim.
\end{proof}

The following example demonstrates that it is sometimes possible to equalize PPV, NPV, FPR, and FNR using deferrals, but without equalizing the
APs themselves:

\begin{example}[Equalizing PPV, NPV, cFPR, and cFNR with Thresholds]
    \label{example:double-threshold-equalization}
This example is presented with continuous support $[0,1]$ for
simplicity.  Consider two APs, one for group $g_1$ and one for $g_2$.
Let the AP for $g_1$ be uniform (with density give by the line
$\pdf(\score) = 1$), and let the AP for group $g_2$ have density
given by the parabola $\pdf(\score) = 6\score(1-\score)$, as shown in Figure \ref{fig:parabola-line}.

Consider the post-processor $\postsoft_{(\tau_0,
  \tau_1)}$.\footnote{In the case where the distributions are continuous,
$\threshRand_0$ and $\threshRand_1$ are meaningless because $\Pr[\score = \tau_0(g)] = \Pr[\score = \tau_1(g)]= 0 \forall \score$.}
Let $\tau_0(g_1) = \tau_0(g_1) = 0.5$, let $\tau_0(g_2) = \frac16(5 - \sqrt{7})$ and let $\tau_1(g_2) = 1 - \frac16(5 -
\sqrt{7})$ as shown in Figure \ref{fig:parabola-line}.

The PPV and NPV of both groups is $\frac34$, and the \cFPR and \cFNR of both is $\frac14$, thus equalizing all four
values.

This example is somewhat unsatisfactory because the base rates are
equal in the two groups. We did not find a similar example without
equal base rates.
\end{example}

\begin{figure}[htp]
    \centering
    \includegraphics[width=0.5\hsize]{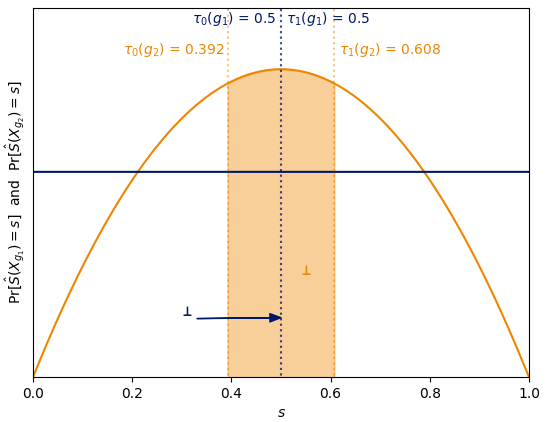}
    \caption{This threshold post-processor equalizes PPV, NPV, cFPR, and cFNR as described in Example
      \ref{example:double-threshold-equalization}. 
  } 
    \label{fig:parabola-line}
\end{figure}

\subsection{Equalizing APs with deferrals}
\label{sec:docs-with-deferrals}
As with Claim~\ref{claim:equalizing-docs-good}, equalizing the conditional APs between groups renders trivial the task of downstream decision-making subject to equality of PPV, NPV, cFPR, and cFNR. Importantly, unlike in Section~\ref{sec:docs-without-deferrals}, equalizing the conditional APs between groups does not require the groups to have equal base rates, greatly increasing the applicability of this approach.

\begin{claim}
  \label{claim:equalizing-conditional-docs-good}
  If the conditional APs are equal for two groups, then PPV, NPV, cFPR, and cFNR are
  equalized (or simultaneously undefined) by any hard deferring
  post-processor $\post$ satisfying (1) \emph{group blindness} and (2)
  $\post(\punt,g) = \punt \   (\forall g)$.
\end{claim}

The additional condition---that $\post$ defers on input $\punt$---is necessary: if $\post$ output $1$ on all inputs (even on $\punt$), then PPV would remain unequal as long as the base rates differed. The proof is similar to the proof of Claim~\ref{claim:equalizing-docs-good}.

 \begin{proof}[Proof of Claim~\ref{claim:equalizing-conditional-docs-good}]
   We prove only that PPV is equalized; the remaining properties may be proved similarly.
   Let $\hard = \post \circ \soft$. All probabilities below are over $X_g\sam \X_g$, and the coins of $\soft$ and $\post$.
   \begin{align*}
      &\PPV_{\hard,g}
      = \Pr[\ty(X_g) = 1 \mid \hard(X_g,g)=1] \\
      &= \frac{\Pr[\ty(X_g) = 1 \mid \soft(X_g) \neq \punt]}{\Pr[\hard(X_g,g) =1\mid \soft(X_g) \neq \punt]} \\
      & \quad \cdot
      \sum_{\score \in \Supp(\pdf_g)} \biggl( \frac{\Pr[\ty(X_g) = 1 \mid \soft(X_g) = \score]\cdot\Pr[\soft(X_g) = \score \mid \soft(X_g) \neq \punt]}{\Pr[\ty(X_g) = 1\mid \soft(X_g) \neq \punt]}\\
      &\quad\quad\quad\quad\quad \cdot \Pr[\post(s,g) = 1 \mid \soft(X_g) = \score, \ty(X_g) = 1]\biggr)
   \end{align*}
   Each factor in this product is equal across groups by the assumptions. Namely, $\Pr[\ty(X_g) = 1\mid \soft(X_g) \neq \punt]$ is equalized by calibration and equalized APs;
   $\Pr[\hard(X_g,g) =1 \mid \soft(X_g) \neq \punt]$ by group blindness and equalized APs;
   $\Pr[\post(s, g) = 1\mid \soft(X_g) = \score, \ty(X_g) = 1]$ by group blindness;
   $\Pr[\ty(X_g) = 1\mid \soft(X_g) = \score]$ by calibration; and finally
   $\Pr[\soft(X_g) = \score \mid \soft(X_g) \neq \punt]$ by equalized APs.
 \end{proof}

Deferrals are a powerful tool for manipulating, and thereby equalizing,  conditional APs. Consider a function $\puntprob:(\score,g) \mapsto [0,1]$
\begin{align*}
    \postsoft_{\puntprob}(\score, g) = \begin{cases}
        \punt & \text{if } \score = \punt \\
        \punt \text{ w.p. } \puntprob(\score, g)\text{, else }\score & \text{otherwise}
    \end{cases}
\end{align*}
If $\soft$ is a calibrated classifier, the soft deferring classifier $\soft' := \postsoft_{\puntprob} \circ \soft$ is still calibrated.
For a group $g$, let $\pdf_g$ be the AP of $\soft$ and $\pdf'_g$ be the AP of $\soft'$.
There is a simple graphical intuition for the shape of $\pdf'_g$, as shown in Figure~\ref{fig:transform_docs1_to_docs2}.
More formally,
\begin{equation}
\label{eq:defer-equal-docs}
\pdf'_g(\score) = \frac{\pdf_g(\score)(1-\puntprob(\score,g))}{1-\Delta}
\end{equation}
where $\Delta := \Pr[\soft'(X_g) = \punt \mid \soft(X_g) \neq \punt] = \sum_{\score\in\Supp(\pdf_g)}\pdf_g(\score)\puntprob(\score,g)$.

\medskip 

By appropriate choice of $\puntprob$, any conditional AP can be transformed into almost any other conditional AP.
\begin{theorem}
  \label{thm:defer-to-any-docs}
  Let $\pdf_g$ be a conditional AP of a soft classifier $\soft$ on group $g$, and let $\pdf^*$ be any probability mass function such that $\Supp(\pdf^*) \subseteq \Supp(\pdf_g)$.
  Then there exists $\puntprob$ for which the calibrated AP $\pdf'_g$ of $\postsoft_{\puntprob}\circ\soft$ is equal to $\pdf^*$.
\end{theorem}
\begin{proof}[Proof of Theorem~\ref{thm:defer-to-any-docs}]
Let $\Delta = 1 - \min_{\score\in \Supp(\pdf_g)} \frac{\pdf_g(\score)}{\pdf^*(\score)}$. For all $\score \in \Supp(\pdf_g)$, let
$$\puntprob(\score,g) = 1 - \frac{\pdf^*(\score)}{\pdf_g(\score)}\cdot(1-\Delta).$$
Observe that
\begin{align*}
\sum_{\score \in \Supp(\pdf_g)}\pdf_g(\score)\puntprob(\score,g) &= \sum_{\score \in \Supp(\pdf_g)} \left( \pdf_g(\score) - (1-\Delta) \pdf^*(\score) \right) \\
&=  \Delta
\end{align*}
and hence $\Delta$ is defined as in Equation~\ref{eq:defer-equal-docs}, where we used the fact that $\sum \pdf_g(\score) = \sum \pdf^*(\score) = 1$ in the last line.
Plugging into the earlier formula for $\pdf'_g$ (Equation~\ref{eq:defer-equal-docs}) completes the proof.
\end{proof}
Together, Theorem~\ref{thm:defer-to-any-docs} and Claim~\ref{claim:equalizing-conditional-docs-good} suggest a general framework for using deferrals to post-process a soft, possibly deferring classifier $\soft$ which is groupwise calibrated into a hard deferring classifier which simultaneously equalizes PPV, NPV, cFPR, and cFNR across groups, as follows.

For each $g \in \G$, let $\pdf_g$ be the conditional AP of $\soft$ for group $g$.
Let $\pdf^*$ be any conditional AP such that $\Supp(\pdf^*) \subseteq \cap_{g\in \G}\Supp(\pdf_g)$. Use Theorem~\ref{thm:defer-to-any-docs} to equalize the conditional APs for all groups $g\in \G$. Then use any hard post-processor $\post$ satisfying the requirements of Claim~\ref{claim:equalizing-conditional-docs-good} to make the ultimate deferring hard classifier.
This method is shown in Figure \ref{fig:transform_docs1_to_docs2}.

This framework allows for enormous flexibility in the choice of both $\pdf^*$ and $\post$, even when considering just two groups $g_1$ and $g_2$. In Figure~\ref{fig:min-pdf}, we illustrate the first step of the framework on a COMPAS dataset using $\min\{\pdf_{g_1},\pdf_{g_2}\}$ as $\pdf^*$, where $g_1$ is African-Americans and $g_2$ is Caucasians. In Figures~\ref{fig:compas-black-to-white} and~\ref{fig:compas-white-to-black} in Section~\ref{sec:experiments}, we also use $\pdf_{g_1}$ and $\pdf_{g_2}$ as $\pdf^*$.

\begin{figure}[htp]
    \centering
    \includegraphics[width=0.5\hsize]{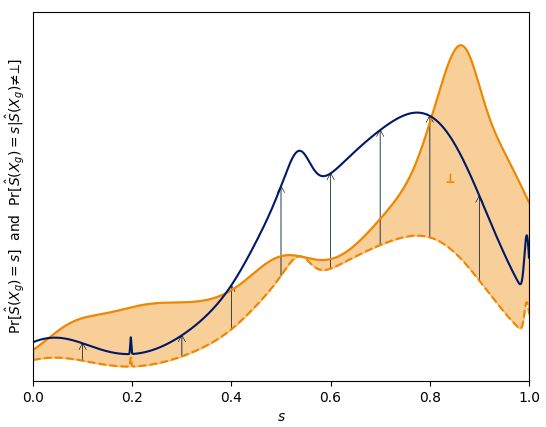}
    \caption{Choosing deferrals cleverly allows transforming one AP into another (conditional) AP.  In this
        example, the solid orange line is the original AP $\pdf_g = \Pr[\soft(X_g) = s]$.  By deferring at the rates
        indicated by the shaded region, the resulting conditional AP $\pdf_g' = \Pr[\soft(X_g)=s \mid \soft(X_g) \ne
    \idk]$ is represented by the dark blue line.  The area of the shaded region is $\Delta$.}
    \label{fig:transform_docs1_to_docs2}
\end{figure}

One can design $\pdf^*$ to achieve additional goals. For example, the choice $\pdf^* = \min\{\pdf_{g_1},\pdf_{g_2}\}$ results in \emph{equal deferral rate} across each group (equal to the total variation distance between the two initial conditional APs).
The framework can be further expanded by combining deferrals with other methods for manipulating  conditional APs, including the mass-averaging discussed in Section~\ref{sec:docs-without-deferrals}.
A better understanding of these techniques is left for future work.



\section{Experiments on COMPAS Data}
\label{sec:experiments}
We test our methodology on the Broward County data made publicly available by ProPublica \cite{angwin2016machine}. This data set contains the recidivism risk decile scores given by the COMPAS tool, 2-year recidivism outcomes, and a number of demographic and crime-related variables on individuals who were scored in 2013 and 2014. We restrict our attention to the subset of defendants whose race is recorded as African-American or Caucasian. These will form the two groups with respect to which we wish to examine different fairness criteria. 
After applying the same data pre-processing and filtering as reported in the ProPublica analysis, we are left with a data set on n = 5278 individuals, of
whom 3175 are African-American and 2103 are Caucasian. \\

\noindent 
Indeed, it has been shown that the COMPAS scoring mechanism is an approximately calibrated soft classifier with 10 possible outcomes. We note here that the distribution of the COMPAS scores differs significantly across the two groups. In particular, the scores for Caucasians are more evenly distributed as opposed to the skewed distribution seen with African-Americans. 

\subsection{Thresholding with Deferrals}
We first ran our two-threshold post-processing mechanism (Section~\ref{sec:threshold-deferrals}) and obtained a binary decision algorithm with deferrals which equalizes both PPV and NPV across Caucasians and African-Americans (See Figure \ref{fig:compas-thresholds}). For simplicity we avoid using randomization for members within a particular decile score and instead settle for approximate equalization of PPV and NPV. We observe that the percent of deferrals in total is smaller than 20\% of the decisions to be made which shows that a fairly large number of the defendants can be classified without having to defer to a downstream decision maker.

\begin{figure}[htp]
	\centering
    \subfloat{\includegraphics[width=0.4\hsize]{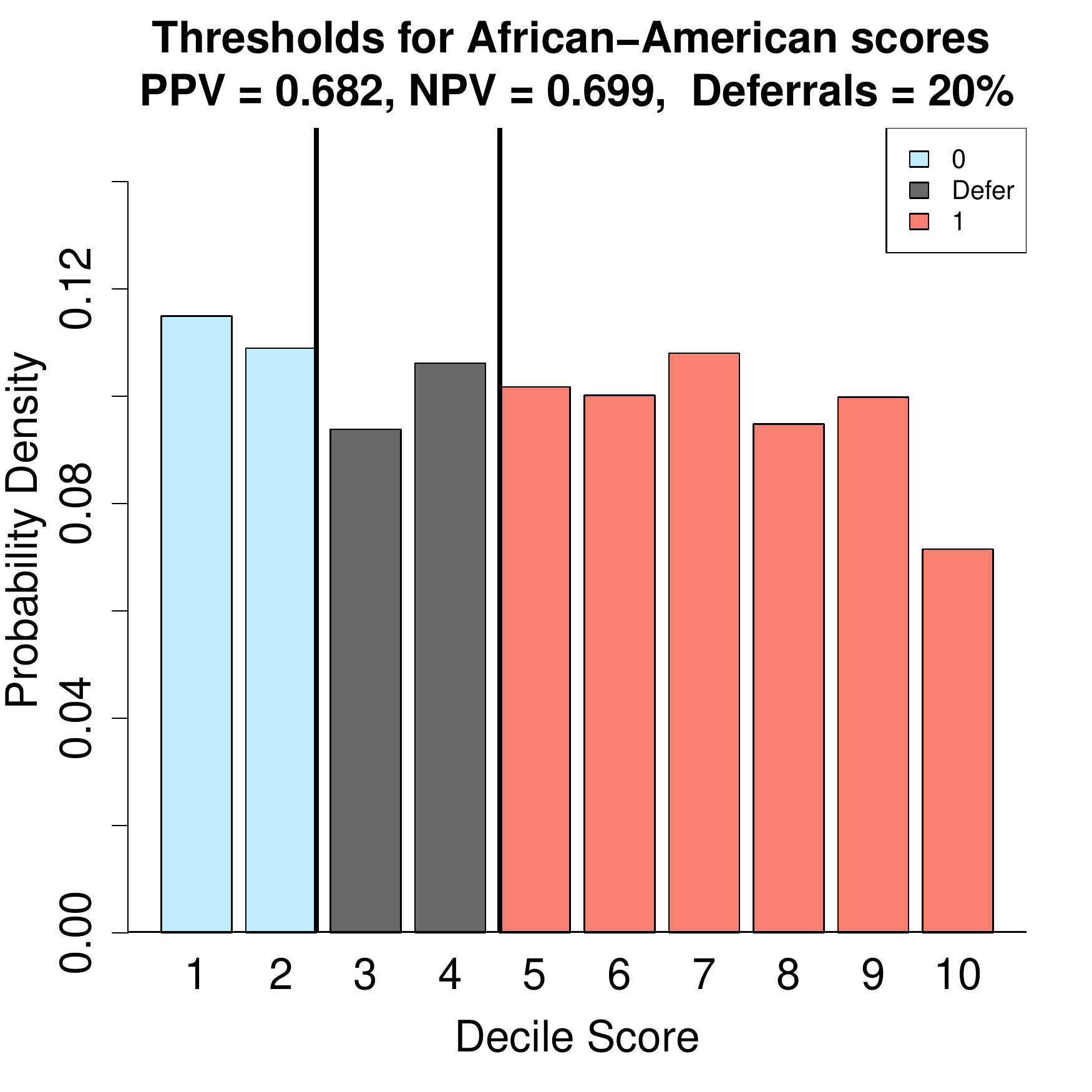}}
	\hspace*{2cm}
    \subfloat{\includegraphics[width=0.4\hsize]{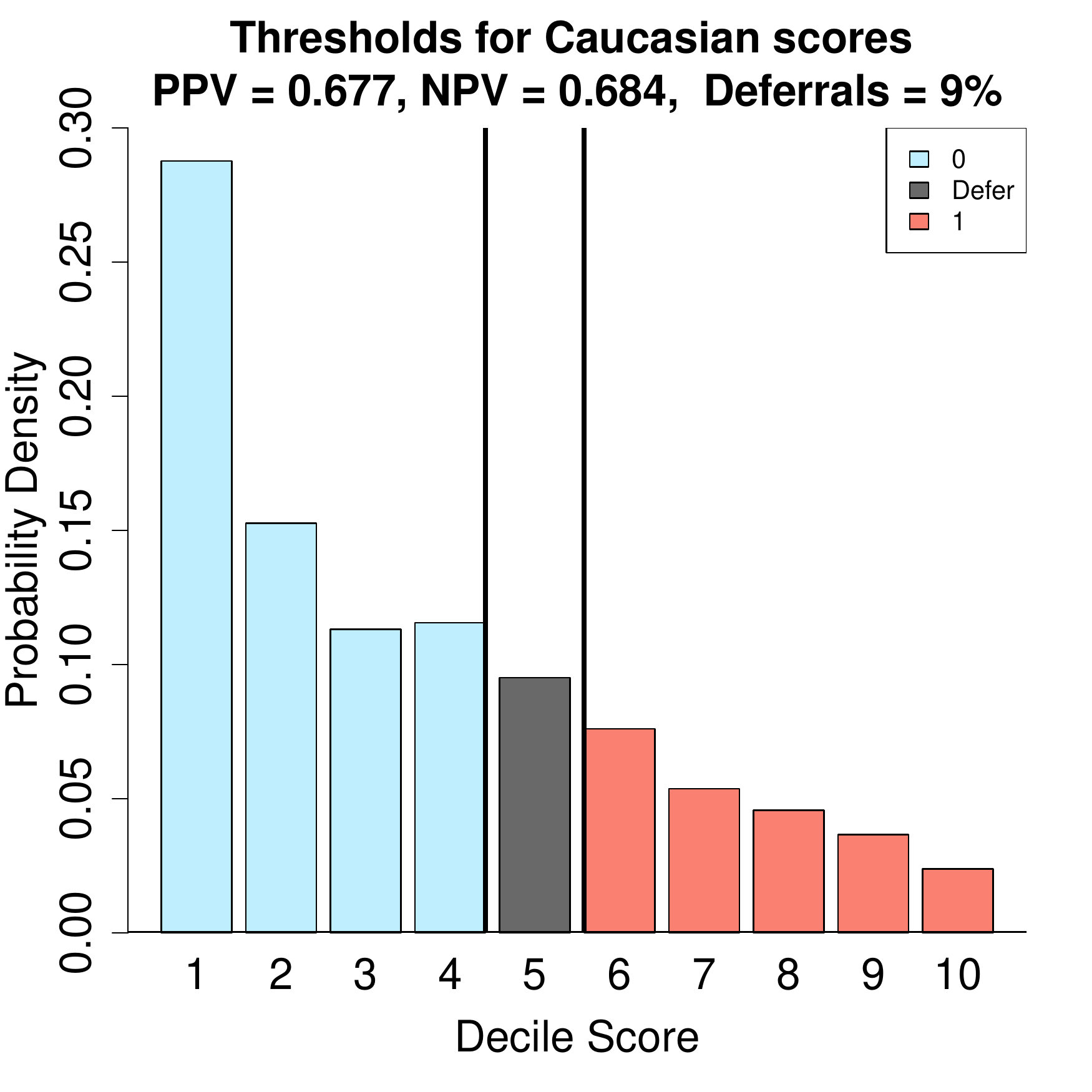}}
    \caption{Two thresholds are applied to each AP for the COMPAS data from 2016, (approximately) equalizing PPV and NPV.  In the left plot we show the thresholds for the African American group, and in the right plot, we show the thresholds for the Caucasian group.}
    \label{fig:compas-thresholds}
\end{figure}

The thresholds suffer from the issue highlighted in Example~\ref{ex:social-unsat}, demonstrating that blindly equalizing PPV and NPV using thresholds can be problematic. Namely, the resulting deferring hard classifier is much stricter for African-Americans than for Caucasians. 

Next we look at our post-processing mechanisms to equalize all four quantities PPV, NPR, FPR, and NPR using deferrals (Section~\ref{sec:docs-with-deferrals}). As was noted earlier in the paper, equalizing the APs of the two groups post-deferral achieves the goal of equalizing all four of the above quantities. We implement two methods for doing so.   
\subsection{Converting one AP into Another}
In the first method, decisions are deferred only on one group so as to convert its AP into that of the other group. First, we consider deferring only on Caucasians to convert their AP into that of African-Americans (Figure~\ref{fig:compas-white-to-black}); next, decisions are deferred only for African Americans (Figure~\ref{fig:compas-black-to-white}).
\begin{figure*}[htp]
	\centering
    \begin{minipage}{.5\textwidth}
    	\centering
  		\includegraphics[width=0.8\linewidth]{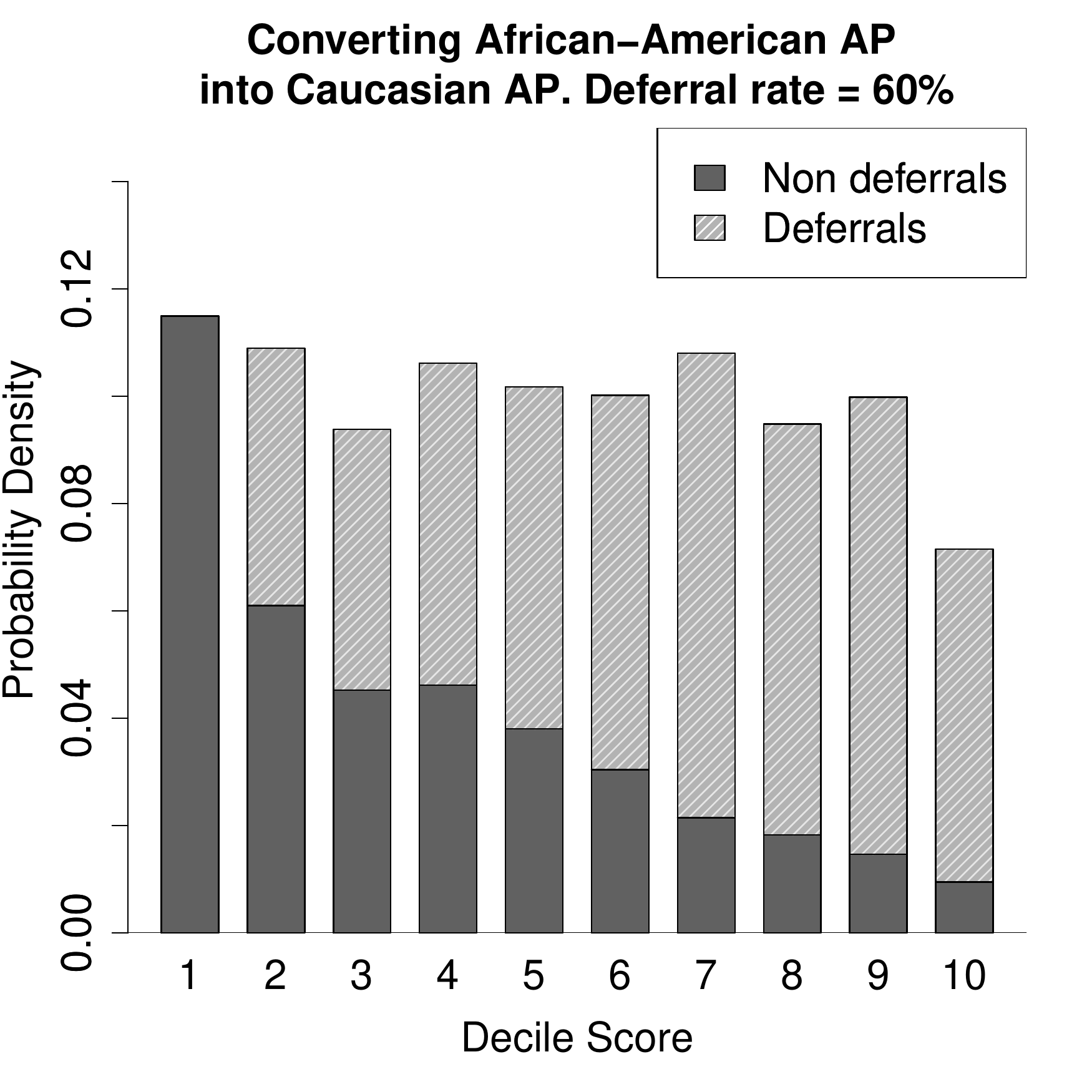}
    	\caption{Our conditional AP equalization method applied to COMPAS data from 2016. We use deferrals to create a conditional AP for African Americans that matches the AP for Caucasians.}
    	\label{fig:compas-black-to-white}
	\end{minipage}%
    \hspace*{1cm}
    \begin{minipage}{.5\textwidth}
    	\centering
        \includegraphics[width=0.8\linewidth]{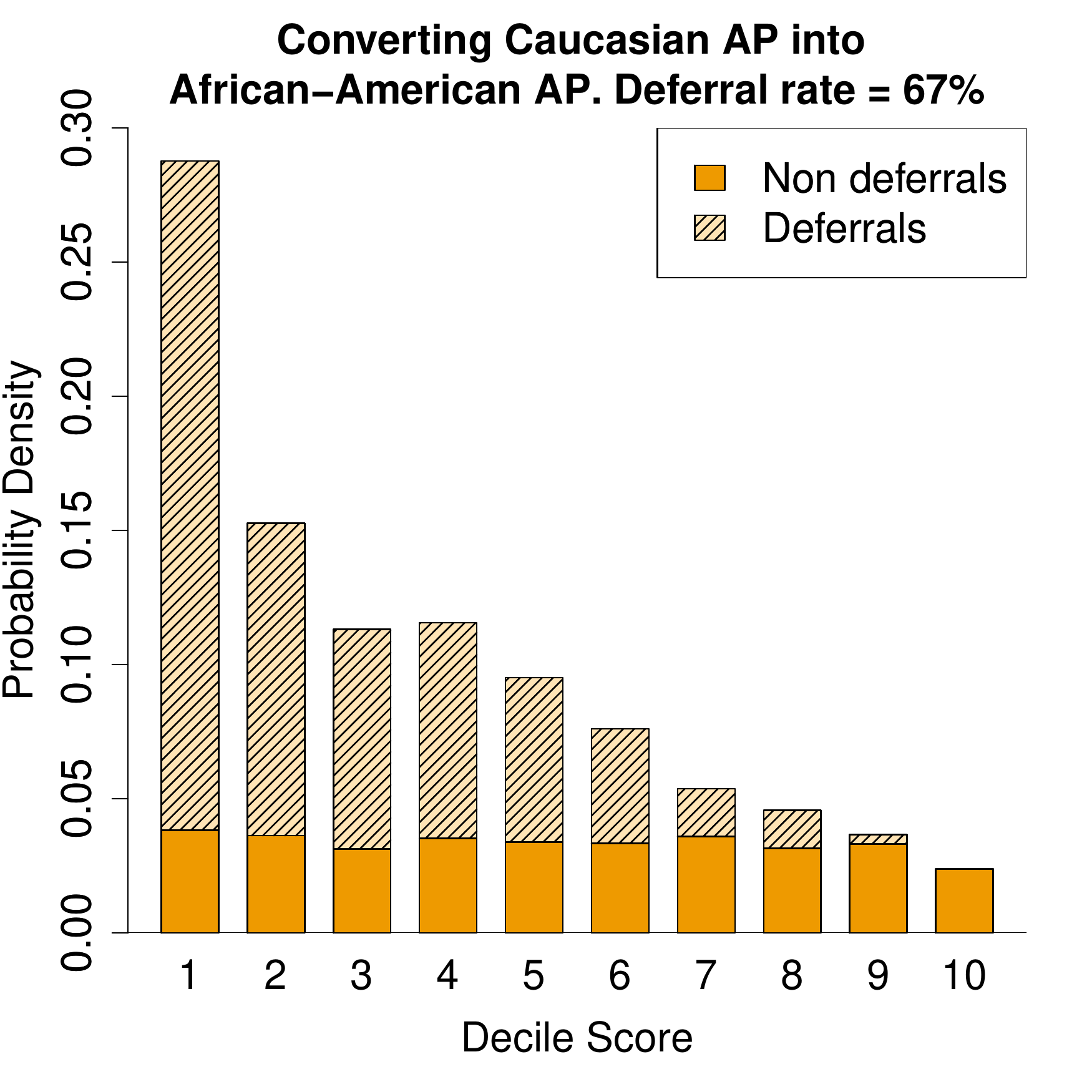}
  		\caption{We create a conditional AP for Caucasians using the AP for African Americans. Notice the difference in the rates and distributions of deferrals between the two methods.}
  		\label{fig:compas-white-to-black}
	\end{minipage}
\end{figure*}

\subsection{Equalizing APs}
Alternately we have a second method where decisions are deferred for an equal fraction of Caucasians and of African Americans (Figure~\ref{fig:min-pdf}). This fraction is precisely equal to the statistical  (\emph{total variation}) distance between the distributions of scores produced by the soft classifier on the two groups.
\begin{figure}[htp]
	\centering
    \includegraphics[width=0.5\hsize]{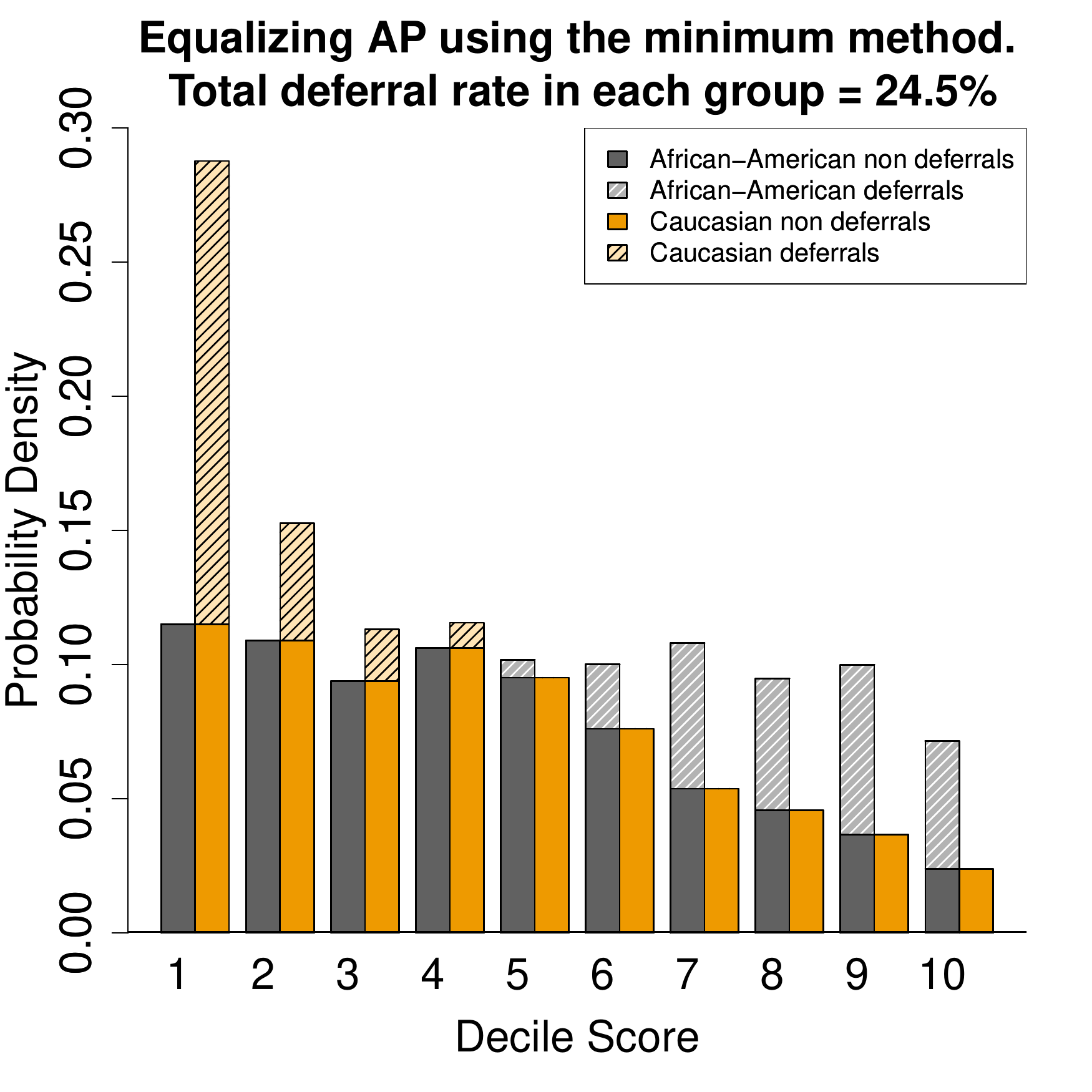}
    \caption{A version of our conditional AP equalization method applied to COMPAS data from 2016.  The AP for African Americans and the AP for Caucasians are converted into a conditional AP that has the same distribution as the pointwise minimum of the two APs.  Notice that the total deferral rate is equalized between the two groups (this is equal to precisely the total variation distance between the two APs), but the distribution of deferrals across scores is not.}
    \label{fig:min-pdf}
\end{figure}

We observe several striking phenomena on the COMPAS data set. First, using the method of deferring only on African-Americans we defer on roughly 36\% of the total decisions. 
This number goes down to roughly 25\% when we defer only on Caucasians. This seems to suggest as a general heuristic to try and use deferrals on the group with smaller size.
The total deferral fraction is also roughly 25\% when we defer on an equal fraction of Caucasians and African-Americans.

Second, for all three  methods that equalize the score distributions,  deferrals  happen  more on the ``extremes'', namely on individuals with respect to which the classifier had relatively high confidence (either close to 0 or close to 1). This stands in sharp contrast to how the two-threshold method (Figure \ref{fig:compas-thresholds}) distributes its deferrals---they occur only in the middle of the distribution (namely for elements for which the classifier is ``unsure''). More work is needed to better understand the full space of deferral strategies.

 \section{Models of Deferring}
\label{sec:discussion}

Whether or not a classifier is thought of as promoting fairness depends on the context; this is true for both deferring and non-deferring classifiers. In addition to the myriad considerations present for non-deferring classifiers, deferring classifiers and downstream decision makers introduce some additional axes for consideration. 
    \paragraph{Cost to the individual:} Even though it is not intended to be a final decision, a deferral may impose burdensome costs to an individual being classified. It may mean that a defendant remains in jail while additional hearings are scheduled, that invasive and expensive medical tests are ordered, or that continued investigation engenders social stigma. These costs may not be borne equally by all individuals, and may depend on their group membership, their true type, or other factors. For example, a delay in granting a loan to a applicant may overly burden poorer applicants, even those very likely to repay. 
    \paragraph{Cost to the decider:} Allowing  deferrals  might make the decision process more cost-effective: Given that in most cases making a determination is cheap, one may now invest more in the deferred cases.  For instance, a team of trained moderators might be hired  to manually review content on which an automated content filter defers, or an expensive investigation might be required to adjudicate insurance claims that are not cut-and-dry. 
    \paragraph{Accuracy of downstream decision} One reason to defer is to introduce a delay that will allow for a more accurate decision. Thus the usefulness of allowing a classifier to defer depends on the accuracy of the downstream decision maker. Additional medical tests might allow for highly accurate diagnoses.  But a judge deciding bail will be prone to a variety of errors and biases. 
\paragraph{``Fairness'' of downstream decision (and of composed classifier)} Similar to the above, the fairness of the downstream decision maker (however one wants to interpret that) will impact our interpretation of the deferring classifier. 
Here one should take into account also the "procedural" aspect of the two-step evaluation; here it is important that the downstream  classifier will be deemed as "more fair"  and  "more knowledgeable" than the first stage. 
Exploring fairness criteria for systems of deferring classifiers and downstream decision-makers, e.g. as done in~\cite{bower} did for non-deferring classifiers, is an interesting direction for future work.
    \paragraph{Frequency of decisions} In many settings, the deferring classifier is a fast, automated test (e.g., automated risk assessment) while the downstream decision maker is a slow, manual process (e.g., parole board). However, we anticipate situations in which there may be repeated deferring classifiers chained together which comprise the complete decision making pipeline. For example, a doctor might have a sequence of diagnostic tests at her disposal as needed, or a bank might allow many rounds of appeal for loan applications, but with lengthy delays. 
Some applications might even permit hundreds or thousands of near-continuous deferring classifiers. As an example, consider a live video streaming platform that passively monitors streams for inappropriate content in real time. The automated passive monitor might decide the content is inappropriate, and shut it down; appropriate, and continue passive monitoring; or suspicious (by deferring), and begin active monitoring by devoting more computing resources or bringing in a human moderator.

\subsection{Technical implications of deferral model}

The contextual considerations discussed above directly impact the appropriate application of a deferring classifier and its goals. An obvious goal is to minimize the overall rate of deferrals while maintaining the best possible FPR, FNR, PPV, and NPV for the classifier conditioned on not deferring, and without considering the distribution of deferrals. However, one might desire very different properties from the distribution of deferrals in different contexts. The deferrals may be distributed differently among individuals with different true type, group membership, or soft-classifier scores, while the burden imposed by deferrals and errors may differ greatly between different populations. 

In a medical diagnosis scenario, a false negative (i.e., failing to diagnose a disease) may have serious consequences, and deferring to run additional non-invasive and inexpensive additional tests may be generally acceptable. On the other hand, an insurance provider may prefer to minimize expensive investigations by paying out more false claims.

The context may also affect the way one defines the deferral analogues of \FPR and \FNR. While calibration, \PPV, and \NPV
apply directly to deferring classifiers, it is not clear how best to generalize the definitions of error rates.
For example, consider false positive rate: by Definition~\ref{defn:four-measures}, the false positive rate of a non-deferring hard classifier $\hard$ for a group $g$ is $\FPR_{\hard,g} = \Pr[\hard(X_g) = \pT \mid \ty(X_g) = \tF]$.

The approach we take in Section~\ref{sec:deferrals} is to condition on not deferring (Definition~\ref{defn:conditional-measures}). A deferring classifier $\hard$ that output $\pT$ on half of true negative instances (within a $g$) would have
conditional false positive rate as low as 0.5 (if it never output $\punt$ on true negatives) or as high as $1$ (if it
never output $\pF$ on true negatives). The conditional false positive rate is agnostic towards the downstream decision maker. It codifies no value judgements as to whether a deferral is desirable or undesirable as an individual nor whether deferrals ultimately result in accurate or inaccurate decisions. This is, itself, a value judgement.

A second approach is to leave the original definition unchanged. The same deferring hard classifier as above would have unconditional false positive rate 0.5. This would be true regardless of whether $\hard$ output $\pF$ or $\punt$ on the other half of true negative instances.
We call this the \emph{unconditional false positive rate}.
The unconditional false positive rate effectively categorizes deferrals as correct outputs. This may be appropriate if the downstream decision maker has very high accuracy. If, for example, a doctor orders an additional, more accurate diagnostic test in response to a deferral, the unconditional false positive rate might be appropriate.

Finally, a third approach is to base our measure of inaccuracy on true negatives instead of false positives, a reverse of the above.

Just as in the case of non-deferring classifiers, the relationships among these contrasting group statistics, their meaningfulness in different settings, and their application in different settings are not well understood and deserve further study.

\bibliographystyle{alpha}
\bibliography{sample}

\newpage
\appendix
\section{Results for NPV and Continuous, Full Support APs}
\subsection{Results for Negative Predictive Value (NPV)}
\label{sec:appendix-npv}

In Section~\ref{sec:postprocess-limits}, we proved limitations on the ability of post-processors to equalize PPV given a distribution on calibrated scores with finite support. For completeness, in this section, we give the statements of the analogous limitations for equalizing NPV and for continuous probability density functions with full support $[0,1]$. 

We start with the analogous statement of Proposition~\ref{prop:info-theory-impossibility-3} for NPV instead of PPV.
\begin{proposition}
    \label{prop:info-theory-impossibility-npv}
Fix two disjoint groups $g_1$ and $g_2$ with respective base rates $\BR_1$ and $\BR_2$ such that $\BR_1 \neq \BR_2$. Then there exists a soft-valued classifier $\soft$ that is groupwise calibrated, but for which there is no post-processor $\post: [0,1] \times \G \to \{0,1\}$ such that $\post \circ \soft$ equalizes $\NPV$, unless $\Pr[\post(BR_i, g_i)=0]=0$ for $i=1$ or $2$.
\end{proposition}

The nontriviality condition ensures that the NPV is well-defined on both groups (which can be compared to the nontriviality condition in Proposition~\ref{prop:info-theory-impossibility-3}, which ensures that the PPV is well-defined on both groups). The proof is essentially identical to the proof of Proposition~\ref{prop:info-theory-impossibility-3}: the fraction of predicted 0's in group $g_i$ that are true $0$'s is $1 - \BR_i$, as the post-processor has no other information by which to make its decision, and hence the NPVs are unequal due to the differing base rates.

We now proceed to give the NPV analogs of our results on threshold post-processors in Section~\ref{sec:postprocess-limits}. We start with the analogous statement for Proposition~\ref{prop:trivial-threshold-equalizes} - that there is a class of simple group-blind threshold post-processors that equalizes the NPV across groups.
\begin{proposition}
\label{prop:trivial-threshold-equalizes-npv}
For every nice groupwise calibrated soft classifier $\soft$ and for every group-blind threshold post-processor $\post_{(\tau, \threshRand)}$ such that $\thresh(g) = \min(\supp(\pdf_g))$ and $\threshRand(g) <1$ for all $g$, the composed classifer $\post \circ \soft$ equalizes $\NPV$s across all groups. 
\end{proposition}
The existence of the threshold post-processors in Proposition~\ref{prop:trivial-threshold-equalizes-npv} follows from the assumed finiteness of the range of the soft classifier. In the case where the range of the soft classifier is infinite, such post-processors may not exist (as there may be no minimum element of the support). The proof is once again analogous to the proof of Proposition~\ref{prop:trivial-threshold-equalizes}: these classifiers only ever output 0 on the minimum element of the support of $\pdf_g$, and hence the NPV is simply 1 minus the smallest element of the support for each group.

However, much like the case for PPV, other group-blind threshold post-processors cannot possibly equalize NPV. 
\begin{proposition}\label{prop:single-threshold-cannot-equalize-npv}
There exists a groupwise-calibrated soft classifier with a nice AP for which no non-trivial group blind threshold post-processor, other than the ones in Proposition~\ref{prop:trivial-threshold-equalizes}, can equalize NPV across groups.
\end{proposition}
The example of the groupwise-calibrated soft classifier is the exact same one that shows the statement for PPV. Indeed, we can see this in the following way. We can make the example in the proof of Proposition~\ref{prop:single-threshold-cannot-equalize-npv} have equal base rates across groups, in which case it follows from Proposition~\ref{prop:single-threshold-cannot-equalize} due to Claim~\ref{claim:baserate-convex-comb}.

When we turn to threshold post-processors that are not group-blind, we again get analogous results for NPV. 

\begin{proposition}\label{prop:2-thresh-equal-npv}
    Let $\G$ be a set of groups.  For any soft classifier $\soft$ with a nice AP $\pdf$ such that
$\soft$ is groupwise-calibrated over $\G$ and
$|\supp(\pdf_{g})| \geq 2$ for all $g \in \G$,
there exists a (\emph{non-group-blind}), non-trivial threshold post-processor $\post_{(\thresh, \threshRand)}$ that is \emph{not} one of the group blind post-processors in Proposition~\ref{prop:trivial-threshold-equalizes}, such that the hard classifier
$\hard = \post_{(\thresh, \threshRand)} \circ \soft$ equalizes NPV across $\G$.

This holds even if we require that the NPV of all the groups is an arbitrary value in $(\score_{min},
\min_i \BR_{g_i})$, where $\min_i \BR_{g_i}$ is the minimum base rate among the groups and $\score_{min}$ is the
minimum score in the support of $\pdf_{g_i}$.\footnote{For the case where the support of $\pdf_{g_i}$ is infinite, $\score_{min}$ should be the infimum of scores.}
\end{proposition}

Again the proof of this follows via the same kind of continuity argument 
as we used to prove Proposition~\ref{prop:2-thresh-equal-ppv}.
By definition, each group has base rate at least $\min_i \BR_{g_i}$, and so if the post-processor 
always says 0 on some group $i$, then $\NPV_{g_i} = \BR_{g_i} \geq \min_i \BR_{g_i}$. Hence, for each group, 
we can start with the always-0 classifier and slide down the threshold until the desired NPV is reached. 

Finally, the socially unsatisfying example also generalizes to NPV. A privileged group will have higher scores than a disadvantaged group in general, and hence if they are given the same threshold, the NPV will be lower on the privileged group. To rectify this, the threshold for the disadvantaged group will have to moved higher, to decrease the \NPV. But then, the disadvantaged group is being subjected to a harsher standard.

Finally we note that Claim~\ref{claim:ppv-continuity} also can be written with NPV instead of PPV, where the proof follows from using
the characterization of NPV given in Proposition~\ref{prop:ppv-as-cond-exp}:
\begin{claim}[Continuity of NPV]
    \label{claim:npv-continuity}
    Fix a soft classifier $\soft$ and a corresponding AP $\pdf$, as well as a group $g$.  Let $\post_1$ and
    $\post_2$ be two post-processing algorithms.  Let $\pdf_{g,\post_1}$ be the expected conditional AP that results
    from starting with the AP $\pdf_g$ over scores in group $g$ and then conditioning on the scores that $\post_1$
    sends to 0, and define $\pdf_{g,\post_2}$ similarly.  If $d_{TV}(\pdf_{g, \post_1}, \pdf_{g, \post_2}) < \epsilon$,
    then $|\NPV_{g, \post_1 \circ \soft} - \NPV_{g, \post_2 \circ \soft}| < O(\epsilon)$.
\end{claim}

We omit the proof of Claim \ref{claim:npv-continuity}, which resembles the proof of Claim \ref{claim:ppv-continuity}
and follows from Proposition \ref{prop:ppv-as-cond-exp}.

\subsection{Results for Continuous, Full Support APs}
\label{sec:appendix-cont}
In this section, we briefly address how to extend our results on thresholds in Section~\ref{sec:postprocess-limits} to the setting where every AP $\pdf$ is a continuous probability distribution with $\supp(\pdf_g) = [0,1]$ for all $g \in G$ - that is, the support equals the entire interval $[0,1]$ for each group $g \in G$. Note that this automatically makes $\pdf$ a ``nice'' AP, and hence rules out the general counterexample we came up with in Proposition~\ref{prop:info-theory-impossibility-3}. For the purposes of this section, call such an AP that 1) $\pdf_g$ is a continuous probability density function for every $g \in G$ and 2) $\supp(\pdf_g) = [0,1]$ for all $g \in G$ a \emph{very nice} AP. As the name suggests, we can extend the remaining results in Section~\ref{sec:postprocess-limits} to the setting of very nice AP. We give the results for equalizing PPV as done in Section~\ref{sec:postprocess-limits}: extending these results to equalizing NPV in the setting of continuous, full support AP can be accomplished by combining the statements here with the modifications described in Section~\ref{sec:appendix-npv}.

First, we note that a threshold post-processor can be described much more easily in the continuous setting than in the setting where the AP has finite support on each group. Indeed, in the setting where $\pdf_g$ is a continuous density function for all $g \in G$, the post-processor can truly be a threshold, with no question of how to classify the score that is exactly equal to the threshold $\tau$. This is because the score $\tau$ has probability 0 under the density $\pdf_g$.

Hence, Proposition~\ref{prop:trivial-threshold-equalizes} has no true analog in this setting. This follows because the maximum element of the support in this case is 1, and a threshold at $\tau=1$ sends every score (outside of a measure 0 set) to 0. 

This allows us to strengthen Proposition~\ref{prop:single-threshold-cannot-equalize} accordingly.
\begin{proposition}\label{prop:single-threshold-cannot-equalize-cont}
There exists a groupwise-calibrated soft classifier with a very nice AP for which no non-trivial group blind threshold post-processor can equalize PPV across groups.
\end{proposition}
This follows from a nearly identical stochastic domination argument to the one used for Proposition~\ref{prop:single-threshold-cannot-equalize} - in fact, the natural generalization of the distributions given for the proof of Proposition~\ref{prop:single-threshold-cannot-equalize} to the continuous and full-support setting can be used in this proof.

A non group blind threshold can still always equalize PPV for very nice APs. 

\begin{proposition}
\label{prop:equalize-ppv-cont}
    Let $\G$ be a set of groups.  For any soft classifier $\soft$ with a very nice AP $\pdf$ such that $\soft$ is
    groupwise-calibrated over $\G$, then
    there exists a non group blind, non-trivial threshold post-processor such that the hard classifier
 equalizes $PPV$ across $\G$.

    This holds even if we require that the PPV of all the groups is equal to an arbitrary value in $(\max_i \BR_{g_i},1)$, where $\max_i \BR_{g_i}$ is the maximum base rate among the groups $g_i \in \G$.
\end{proposition}

This follows from the same continuity approach as the proof of Proposition~\ref{prop:2-thresh-equal-ppv}, by sliding the threshold continuously down from 1 until the PPV reaches the desired value $v \in (\max_i \BR_{g_i},1)$. 

The socially unsatisfying example generalizes in the natural way - Example~\ref{ex:social-unsat} consists of one group having a monotonically increasing PMF and another one having a monotonically decreasing PMF. We can skip the discretization step in the definition of these PMFs and have them be continuous PDFs, and the example still goes through.

We cannot equalize PPV and NPV simultaneously in general, just like in the finite support case (Proposition~\ref{prop:ppv-npv-impossibility}).
\begin{proposition}
Fix groups $g_1$ and $g_2$. There exists a soft classifier $\soft$ with a very nice AP~$\pdf$ such that no threshold post-processor can simultaneously equalize PPV and NPV between groups $g_1$ and $g_2$.
\end{proposition}
The proof we give of Proposition~\ref{prop:ppv-npv-impossibility} for finite support naturally generalizes to this case - in fact, we can simply use the same proof but without discretizing the probability distributions. The necessary lemmas about monotonicity of PPV and being able to express the base rate as a convex combination of PPV and NPV still hold.

It is unsurprising that the result in Section~\ref{sec:deferrals} on using thresholds \emph{with deferrals} to equalize PPV and NPV also goes through for very nice AP.

\begin{proposition}
  Let $\soft$ be groupwise calibrated for the $n$ groups $g_1,\ldots,g_n$, and suppose that $\soft$ has a very nice AP.  Then there exists a nontrivial threshold decision rule rule such that the hard classifier $\hard = \post \circ \soft$
    equalizes PPV and NPV for $\G$. 
 \end{proposition}
Again, the explanation is very similar to the one for Proposition~\ref{prop:equalize-ppv-cont}. PPV and NPV change continuously when we slide the respective thresholds, and unlike the case without deferrals, we can change the PPV \emph{without} changing the NPV, by keeping the ``0'' threshold still and sliding the ``1'' threshold (and deferring in the middle). Hence, we can simply continuously slide the two thresholds on each group until they reached the desired values.

\end{document}